\documentclass[journal]{IEEEtran}

\title{Robust Low-Tubal-Rank Tensor Completion based on Tensor Factorization and Maximum Correntopy Criterion\thanks{This work was supported by NSF CAREER Award CCF-1552497 and NSF Award CCF-2106339.

Yicong He and George K. Atia are with the Department of Electrical and Computer
Engineering, University of Central Florida, Orlando, FL, 32816, USA. e-mails: Yicong.He@ucf.edu, George.Atia@ucf.edu.}}

\author{Yicong He, George K. Atia, \IEEEmembership{Senior~Member,~IEEE}}

\usepackage{graphicx}
\usepackage{amssymb}
\usepackage{algorithm}
\usepackage{algorithmic}
\usepackage{amsmath} 
\usepackage{amsthm}
\usepackage{color}
\usepackage{booktabs}
\usepackage{threeparttable}
\usepackage{multirow}
\usepackage{tabularx}
\usepackage[normalem]{ulem}
\useunder{\uline}{\ul}{}
\newtheorem{theorem}{Theorem}
\newtheorem{proposition}{Proposition}
\newtheorem{remark}{Remark}

\newtheorem {data model}{Data Model}

\newtheorem{lemma}{Lemma}
\newtheorem{mydef}{Definition}
\newcommand{\PreserveBackslash}[1]{\let\temp=\\#1\let\\=\temp}
\newcolumntype{C}[1]{>{\PreserveBackslash\centering}p{#1}}

\begin{document}

\maketitle

\begin{abstract}
The goal of tensor completion is to recover a tensor from a subset of its entries, often by exploiting its low-rank property. Among several useful definitions of tensor rank, the low-tubal-rank was shown to give a valuable characterization of the inherent low-rank structure of a tensor. While some low-tubal-rank tensor completion algorithms with favorable performance have been recently proposed, these algorithms utilize second-order statistics to measure the error residual, which may not work well when the observed entries contain large outliers. In this paper, we propose a new objective function for low-tubal-rank tensor completion, which uses correntropy as the error measure to mitigate the effect of the outliers. To efficiently optimize the proposed objective, we leverage a half-quadratic minimization technique whereby the optimization is transformed to a weighted low-tubal-rank tensor factorization problem. Subsequently, we propose two simple and efficient algorithms to obtain the solution and provide their convergence and complexity analysis. Numerical results using both synthetic and real data demonstrate the robust and superior performance of the proposed algorithms.
\end{abstract}


\section{Introduction}
High-dimensional and multi-way data processing have received considerable attention in recent years given the ever-increasing amount of data with diverse modalities generated from different kinds of sensors, networks and systems. Since tensors are algebraic objects that can be represented as multi-dimensional arrays (generalizing scalars, vectors and matrices), they have marked ability to characterize multi-way (high order) data and capture intrinsic correlations across its different dimensions. This fact explains their wide usage and efficacy in numerous applications of computer vision \cite{dian2017hyperspectral,zhang2019computational}, pattern recognition \cite{xie2018unifying,zhao2019multi,he2006tensor,xie2020robust,tang2021one} and signal processing \cite{sidiropoulos2017tensor,cichocki2015tensor,tang2019tensor}.



Similar to matrices, the data represented by tensors may contain redundant information, which is referred to as the low-rank property of tensors. To exploit the underlying low-rank structure of high order tensors, several low-rank tensor models have been proposed based on different tensor decompositions, including CANDECOMP/PARAFAC (CP) decomposition \cite{kiers2000towards}, Tucker decomposition \cite{tucker1966some}, tensor ring decomposition \cite{zhao2016tensor} and tensor singular value decomposition (t-SVD) \cite{kilmer2011factorization}.

Tensor completion, a generalization of the popular matrix completion problem \cite{candes2009exact,candes2010matrix}, is the the task of filling in the missing entries of a partially observed tensor, typically by exploiting the low-rank property of the tensor. There exist several tensor completion algorithms tailored to different low-rank tensor models, such as the CANDECOMP/PARAFAC decomposition-based alternating minimization algorithm \cite{kiers2000towards,jain2014provable}, Tucker decomposition-based tensor completion using the Riemannian manifold approach \cite{tucker1966some,kasai2016low} and alternating minimization \cite{xu2015parallel}, the t-SVD-based completion algorithm using convex relaxation \cite{zhang2016exact}, alternating minimization \cite{zhou2017tensor,liu2019low} and Grassmannian optimization \cite{gilman2020grassmannian}.

\subsection{Robust Tensor Completion}
In the real world, the data observed could be perturbed by different kinds of noise originating from human errors and/or signal interference. Existing algorithms largely utilize the second-order statistics as their error measure, which works well in certain noisy settings, such as with noise from a Gaussian distribution. However, when the data is contaminated with large outliers, the performance of traditional algorithms is unsatisfactory in general. This motivated the development of robust algorithms for low-rank tensor recovery that are not unduly affected by the outliers \cite{goldfarb2014robust,han2018generalized,inoue2009robust}. While many such algorithms presume that all the entries of the tensor data are observed, several algorithms were designed to deal with incomplete or grossly corrupted data, which is the main focus of this work. 

The vast majority of existing robust tensor completion algorithms are based on tensor rank models that are different from the tubal rank model considered herein. In \cite{huang2014provable}, an $\ell_1$-norm regularized sum of nuclear norms (SNN-L1) completion algorithm is proposed. Rather than directly applying complex Tucker decomposition, which decomposes the tensor into a set of matrices and one small core tensor, SNN-L1 relaxes the low-tucker-rank constraint using a (weighted) sum of nuclear norms of tensor unfolding matrices. Using a similar convex relaxation of Tucker decomposition, a robust low-tucker-rank tensor completion algorithm that uses soft/hard thresholding (SNN-ST/HT) was developed in \cite{yang2015robust}. It introduces two M-estimators, the Welsch loss and the Cauchy loss, as error measures, which improves on SNN-L1. In the same vein, the authors in  \cite{huang2020robust} developed a robust $\ell_1$-norm regularized tensor ring nuclear norm (TRNN-L1) algorithm based on the tensor ring model. Similar to SNN, TRNN-L1 solves the complex tensor ring decomposition problem by minimizing the nuclear norm of circular unfolding matrices. 

{Recently, the tubal rank tensor model has been introduced in the context of tensor completion  \cite{kilmer2013third}. The tubal-rank characterization is based on tensor singular value decomposition (t-SVD) \cite{kilmer2011factorization}, which rests on the tensor-tensor product (t-product) as an extension from matrix algebra. Compared with other tensor models, the tubal rank model has provable theoretical guarantees \cite{kilmer2013third}. It also maintains the intrinsic multi-dimensional tensor structures as it relates to tensor factorization as a product of tensors. Its superior performance over alternate rank models in tensor completion has been established in various works 
\cite{zhang2016exact,zhou2017tensor,liu2019low}. We will further discuss the advantages of low-tubal-rank based tensor completion in Section \ref{sec:LTR_TC}. Several low-tubal-rank based tensor completion algorithms have been developed in the literature, including tensor nuclear norm (TNN) \cite{zhang2016exact} and its robust $\ell_1$-norm regularized version TNN-L1 \cite{jiang2019robust,wang2019robust}.}

Table \ref{tb:algorithm} summarizes the above-mentioned robust tensor completion algorithms along with the tensor rank model they adopt and the corresponding objective functions. As shown, existing robust algorithms utilize the matrix nuclear norm for regularization, which requires performing an SVD in every iteration. {For large matrices, SVD incurs a high computational cost. Further, because of the complex computation of an SVD on a large matrix, algorithms such as TNN-L1 are not readily amenable to parallel implementation on GPU.
}

\begin{table}[tb]
\renewcommand\arraystretch{1.5}
\centering
\caption{Objective functions of robust tensor completion algorithms}
\label{tb:algorithm}
\resizebox{0.5\textwidth}{!}
{%
\begin{tabular}{ccc}
\toprule
Algorithm     & Rank model  & Objective function
\\ \midrule
\begin{tabular}[c]{@{}c@{}}SNN-L1\\\cite{goldfarb2014robust}\end{tabular}        & Tucker      &  \begin{tabular}[c]{@{}c@{}}$\min\limits_{\boldsymbol{\mathcal{X}}, \boldsymbol{\mathcal{S}}}\sum\limits_{i=1}^{N}\left\|\boldsymbol{X}_{i,(i)}\right\|_{*}+\lambda\|\boldsymbol{\mathcal{S}}\|_{1}$\\ $\text { s.t. } \boldsymbol{\mathcal{P}}\circ\left(\sum\limits_{i=1}^{N} \boldsymbol{\mathcal{X}}_{i}+\boldsymbol{\mathcal{S}}\right)=\boldsymbol{\mathcal{P}}\circ{\boldsymbol{\mathcal{M}}}$\end{tabular}    \\    \midrule
\begin{tabular}[c]{@{}c@{}}SNN-ST\\\cite{yang2015robust}\end{tabular}   & Tucker      & \begin{tabular}[c]{@{}c@{}}$\min\limits_{\boldsymbol{\mathcal{X}}} \!\!\sum\limits_{i_1\ldots i_N\in\Omega}\rho_{\sigma}(({\mathcal{M}}_{i_1\ldots i_N}\!-\!(\!\!\sum\limits_{m=1}^{N} \!\!\boldsymbol{\mathcal{X}}_{m}\!)_{i_1\ldots i_N})\!)$   \\ $+\lambda \sum\limits_{j=1}^{N}\left\|\boldsymbol{X}_{j,(j)}\right\|_{*}$\end{tabular}    \\   \midrule
\begin{tabular}[c]{@{}c@{}}SNN-HT\\\cite{yang2015robust}\end{tabular}   & Tucker      & \begin{tabular}[c]{@{}c@{}}$\min\limits_{\boldsymbol{\mathcal{X}}} \!\!\sum\limits_{i_1\ldots i_N\in\Omega}\rho_{\sigma}(({\mathcal{M}}_{i_1\ldots i_N}\!-\!(\!\!\sum\limits_{m=1}^{N} \!\!\boldsymbol{\mathcal{X}}_{m}\!)_{i_1\ldots i_N})\!)$   \\ $+ \sum\limits_{j=1}^{N}\delta_{M_{r_j}}\!\!\left(\boldsymbol{X}_{j,(j)}\right)$\end{tabular}     \\   \midrule
\begin{tabular}[c]{@{}c@{}}TRNN-L1\\\cite{huang2020robust}\end{tabular}       & Tensor ring & \begin{tabular}[c]{@{}c@{}}$\min\limits_{\boldsymbol{\mathcal{X}}, \boldsymbol{\mathcal{S}}} \sum\limits_{d=1}^{N} w_{d}\left\|\boldsymbol{X}_{\{d, L\}}\right\|_{*}+\lambda_{d}\|\boldsymbol{\mathcal{S}}\|_{1}$\\ $\text { s.t. } \boldsymbol{\mathcal{P}}\circ\left(\boldsymbol{\mathcal{X}}+\boldsymbol{\mathcal{S}}\right)=\boldsymbol{\mathcal{P}}\circ{\boldsymbol{\mathcal{M}}}$\end{tabular}
    \\   \midrule
\begin{tabular}[c]{@{}c@{}}TNN-L1\\\cite{jiang2019robust,wang2019robust}\end{tabular}        & Tubal       & \begin{tabular}[c]{@{}c@{}}$\min\limits _{\boldsymbol{\mathcal{X}} . \boldsymbol{\mathcal{S}}}\frac{1}{n_{3}} \sum\limits_{i=1}^{n_{3}}\left\|\bar{\boldsymbol{X}}^{(i)}\right\|_{*}+\lambda\|\boldsymbol{\mathcal{S}}\|_{1}$\\ $\text { s.t., } \boldsymbol{\mathcal{P}}\circ\left(\boldsymbol{\mathcal{X}}+\boldsymbol{\mathcal{S}}\right)=\boldsymbol{\mathcal{P}}\circ{\boldsymbol{\mathcal{M}}}$\end{tabular}     \\   \midrule
\begin{tabular}[c]{@{}c@{}}HQ-TCTF\\ /HQ-TCASD\\
(This paper)\end{tabular} & Tubal  &       $\min\limits_{{\boldsymbol{\mathcal{X}}, \boldsymbol{\mathcal{Y}}}} \!\sum\limits_{i,j,k}\!  \!{{\cal{P}}_{ijk}\sigma^2\!\left({\!1\!-\!G_{\sigma}\!{{\left( {{{\cal{M}}_{ijk}} \!-\! {\left( {{\boldsymbol{\cal{X}}\!*\!\boldsymbol{{\cal{Y}}} }}\right)_{ijk}}}\! \right)}}}\! \right)}$     \\ \bottomrule
\end{tabular}%
}
\end{table}

\subsection{Contribution}
{In sharp contrast to the foregoing work, we propose a novel SVD-free {and parallelizable} robust tensor completion method based on tensor factorization and the maximum correntropy criterion \cite{liu2007correntropy,chen2016generalized} under the tubal rank model. 
Tensor factorization is theoretically grounded on the fact that a best tubal rank-$r$ approximation can be obtained from truncation of the t-SVD.
Further, algorithms based on tensor factorization  (as opposed to minimizing norms of tensor unfoldings) 
were shown to yield accurate performance  \cite{zhou2017tensor,liu2019low}.} Correntropy is an information-theoretic 
non-linear similarity measure that can provably handle the negative effect of large outliers \cite{he2010maximum,he2019robust,zheng2020broad}. Compared with the commonly used $\ell_1$-norm, correntropy is everywhere differentiable and is at the heart of several robust algorithms in different fields \cite{chen2017maximum,zhao2011kernel}. By introducing correntopy as our error measure, we propose a novel correntropy-based objective function for robust low-tubal-rank tensor completion. To efficiently solve the formulated completion problem, we first leverage a half-quadratic (HQ) optimization technique \cite{nikolova2005analysis} to transform the non-convex problem to a weighted tensor factorization problem. Then, two efficient and simple algorithms based on alternating minimization and alternating steepest descent are developed, and 
we analytically establish the convergence of both algorithms. Also, we propose an adaptive kernel width selection strategy to further improve the convergence rate and accuracy. The main contributions of the work are summarized below.

%
{1)} We propose a novel objective function for robust low-tubal-rank tensor completion, which uses tensor factorization to capture the low-rank structure and correntropy as the error measure to give robustness against outliers. As shown in Table \ref{tb:algorithm}, our approach imposes the low-rank structure through factorization. Compared with other existing nuclear-norm-based robust tensor completion algorithms, our factorization-based method does not need to perform multiple SVD computations. 

2) We reformulate the complex correntropy-based optimization problem as a weighted tensor factorization by leveraging the HQ minimization technique (Section \ref{sec:HC}). We develop two efficient algorithms (HQ-TCTF and HQ-TCASD) for robust tensor completion (See Section \ref{sec:HQ-TCTF} and \ref{sec:HQ-TCASD}). The algorithms utilize alternating minimization and alternating steepest descent, which avoid the costly computation of the SVD operations and are {readily parallelizable on GPU}. We also analyze the convergence and computational complexity of the proposed algorithms. 

3) We demonstrate the robust and efficient performance of the proposed algorithms through extensive numerical experiments performed with both synthetic and real data. {The proposed methods can outperform nuclear-norm-based methods in many noisy settings in terms of PSNR. With the use of parallel computation, the proposed methods can also run significantly faster than other algorithms.} 
%

The paper is organized as follows. In Section \ref{sec:bkgnd}, we introduce our notation and provide some preliminary background on the tensor properties, tensor completion, and the maximum correntropy criterion. In Section \ref{sec:methods}, we propose the new correntropy-based tensor completion cost and propose two HQ-based algorithms. In Section \ref{sec:results}, we present experimental results to demonstrate the reconstruction performance. Finally, conclusion is given in Section \ref{sec:conc}.

\section{Preliminaries}
\label{sec:bkgnd}

\subsection{Definitions and Notation}
\label{sec:prelim}
In this section, we review some important definitions and introduce notation used throughout the paper. Boldface uppercase script letters are used to denote tensors (e.g., $\boldsymbol{\cal{X}}$), and boldface letters to denote matrices (e.g., ${\boldsymbol{X}}$). Unless stated otherwise, we focus on third order tensors, i.e., ${\boldsymbol{\cal{X}}}\in{\mathbb{C}}^{n_1\times n_2\times n_3}$ where $n_1,n_2,n_3$ are the dimensions of each way of the tensor. The notation ${\boldsymbol{\cal{X}}}(i,:,:),{\boldsymbol{\cal{X}}}(:,i,:),{\boldsymbol{\cal{X}}}(:,:,i)$ denotes the frontal, lateral, horizontal slices of $\boldsymbol{\cal{X}}$, respectively, and  ${\boldsymbol{\cal{X}}}(i,j,:),{\boldsymbol{\cal{X}}}(:,j,k),{\boldsymbol{\cal{X}}}(i,:,k)$ denote the mode-1, mode-2, and mode-3 tubes, while ${\cal{X}}_{ijk}$ denotes the $(i,j,k)$-th entry of tensor $\boldsymbol{\cal{{X}}}$. The Frobenius norm of tensor is defined as $\|{\boldsymbol{\cal{X}}}\|_F=\sqrt{\sum_{i=1}^{n_1}\sum_{j=1}^{n_2}\sum_{k=1}^{n_3}|{{\cal{X}}_{ijk}}|^2}$.

In the frequency domain, $\bar{\boldsymbol{\cal{X}}}$ denotes the Fourier transform along the third mode of $\boldsymbol{\cal{X}}$. We use the convention, $\bar{\boldsymbol{\cal{X}}}=\operatorname{fft}({\boldsymbol{\cal{X}}},[\:],3)$ to denote the Fourier transform along the third dimension. Similarly, we use ${\boldsymbol{\cal{X}}}=\operatorname{ifft}({\bar{\boldsymbol{\cal{X}}}},[\:],3)$ for the inverse transform. We also define the matrix ${\bar{\boldsymbol{X}}}\in{\mathbb{R}}^{n_1n_3\times n_2n_3}$
\[
\bar{\boldsymbol{X}}=\operatorname{bdiag}(\bar{\boldsymbol{\cal{X}}})=\left[\begin{array}{cccc}
\bar{\boldsymbol{X}}^{(1)} & & \\
& \bar{\boldsymbol{X}}^{(2)} & \\
& & \ddots & \\
& & & \bar{\boldsymbol{X}}^{\left(n_3\right)}
\end{array}\right]
\]
where ${\boldsymbol{X}}^{(i)}:={\boldsymbol{\cal{X}}}(:,:,i)$, and $\operatorname{bdiag}(\cdot)$ denotes the operator that maps the tensor $\bar{\boldsymbol{\cal{X}}}$ to the block diagonal matrix $\bar{\boldsymbol{X}}$. The block circulant operator $\operatorname{bcirc}(\cdot)$ is defined as

\[
\operatorname{bcirc}({\boldsymbol{\cal{X}}})=\left[\begin{array}{cccc}
{\boldsymbol{X}}^{(1)} & {\boldsymbol{X}}^{\left(n_3\right)} & \cdots & {\boldsymbol{X}}^{(2)} \\
{\boldsymbol{X}}^{(2)} & {\boldsymbol{X}}^{(1)} & \cdots & {\boldsymbol{X}}^{(3)} \\
\vdots & \vdots & \ddots & \vdots \\
{\boldsymbol{X}}^{\left(n_3\right)} & {\boldsymbol{X}}^{\left(n_3-1\right)} & \cdots & {\boldsymbol{X}}^{(1)}
\end{array}\right]\:.
\]
Therefore, the following relation holds, 
\begin{equation}
\left(\boldsymbol{F}_{n_{3}} \otimes \boldsymbol{I}_{n_{1}}\right) \operatorname{bcirc}(\boldsymbol{\mathcal { X }}) \left(\boldsymbol{F}_{n_{3}}^{-1} \otimes \boldsymbol{I}_{n_{2}}\right)=\bar{\boldsymbol{X}}\:,
\label{eq:prelim_1}
\end{equation}
where $\boldsymbol{F}_{n_{3}}\in\mathbb{C}^{n_3\times n_3}$ is the Discrete Fourier Transform (DFT) matrix, $\otimes$ is the Kronecker product and $\boldsymbol{I}_{n_1}\in\mathbb{R}^{n_1\times n_1}$ is the identity matrix. Further, $\boldsymbol{F}_{n_{3}}^{-1}$ can be computed as $\boldsymbol{F}_{n_3}^{-1}=\boldsymbol{F}_{n_3}^{*}/n_3$, where $\boldsymbol{X}^*$ denotes the Hermitian transpose of $\boldsymbol{X}$.

To define the tensor-tensor product (t-product), we first define the unfold operator $\operatorname{unfold}(\cdot)$, which maps the tensor ${\boldsymbol{\cal{X}}}$ to a matrix $\tilde{\boldsymbol{X}}\in\mathbb{C}^{n_1n_3\times n_2}$,
\[
\tilde{\boldsymbol{X}}=\operatorname{unfold}(\boldsymbol{\mathcal{X}})=\left[\begin{array}{c}
{\boldsymbol{X}}^{(1)} \\
{\boldsymbol{X}}^{(2)} \\
\vdots \\
{\boldsymbol{X}}^{\left(n_{3}\right)}
\end{array}\right]
\]
and its inverse operator $\operatorname{fold}(\cdot)$ is defined as
\[\operatorname{fold}(\tilde{\boldsymbol{X}})=\boldsymbol{\mathcal{X}}\:.
\]
We can readily state the definition of the t-product.
\begin{mydef}
[t-product \cite{kilmer2011factorization}] The t-product $\boldsymbol{\mathcal{A}} * \boldsymbol{\mathcal{B}}$ of $\boldsymbol{\mathcal{A}} \in$ $\mathbb{R}^{n_{1} \times n_{2} \times n_{3}}$ and $\boldsymbol{\mathcal{B}} \in \mathbb{R}^{n_{2} \times n_{4} \times n_{3}}$ is the tensor of size $n_1\times n_4 \times n_3$ given by
\[
\boldsymbol{\mathcal{A}} * \boldsymbol{\mathcal{B}}=\operatorname{fold}(\operatorname{bcirc}(\boldsymbol{\mathcal{A}}) \cdot \operatorname{unfold}(\boldsymbol{\mathcal{B})})\]
\label{def:tprod}
\end{mydef}
Further, we will need the following lemma from \cite{kilmer2011factorization}.
\begin{lemma}
\cite{kilmer2011factorization} Suppose that $\boldsymbol{\mathcal{A}} \in \mathbb{R}^{n_{1} \times n_{2} \times n_{3}}, \boldsymbol{\mathcal{B}} \in$
$\mathbb{R}^{n_{2} \times n_{4} \times n_{3}}$ are two arbitrary tensors. Let $\boldsymbol{\mathcal{F}}=\boldsymbol{\mathcal{A}} *\boldsymbol{\mathcal{B}} .$ Then,
the following properties hold.

(1) $\|\boldsymbol{\mathcal{A}}\|_{F}^{2}=\frac{1}{n_{3}}\|\bar{\boldsymbol{A}}\|_{F}^{2}$

(2) $\boldsymbol{\mathcal{F}}=\boldsymbol{\mathcal{A}} * \boldsymbol{\mathcal{B}}$ and $\bar{\boldsymbol{F}}=\bar{\boldsymbol{A}} \bar{\boldsymbol{B}}$ are equivalent.
\label{lem:t-product}
\end{lemma}
According to the second property in Lemma \ref{lem:t-product}, the t-product is equivalent to matrix multiplication in the frequency domain.

Next, we state the definitions of the Tensor Singular Value Decomposition (t-SVD) and the tubal rank.
\begin{theorem}
[t-SVD
\cite{kilmer2011factorization,lu2019tensor}] The tensor $\boldsymbol{\mathcal{A}} \in \mathbb{R}^{n_{1} \times n_{2} \times n_{3}}$ can be factorized as $\boldsymbol{\mathcal{A}}=\boldsymbol{\mathcal{U}} * \boldsymbol{\mathcal{S}} * \boldsymbol{\mathcal{V}}^{*}$, where $\boldsymbol{\mathcal{U}} \in \mathbb{R}^{n_{1} \times n_{1} \times n_{3}}, \boldsymbol{\mathcal{V}} \in \mathbb{R}^{n_{2} \times n_{2} \times n_{3}}$ are orthogonal, and $\boldsymbol{\mathcal{S}} \in$
$\mathbb{R}^{n_{1} \times n_{2} \times n_{3}}$ is an $f$-diagonal tensor, i.e., each of the frontal slices of $\boldsymbol{\mathcal{S}}$ is a diagonal matrix. The diagonal entries in $\boldsymbol{\mathcal{S}}(:,:,1)$ are called the singular values of $\boldsymbol{\mathcal{A}}$.
\end{theorem}

\begin{mydef}
[Tensor tubal-rank \cite{kilmer2013third}] For any $\boldsymbol{\mathcal{A}} \in \mathbb{R}^{n_{1} \times n_{2} \times n_{3}},$ the tensor tubal-rank, rank$_{t}(\boldsymbol{\mathcal{A}}),$ is the number of non-zero singular tubes of $\boldsymbol{\mathcal{S}}$ from the t-SVD, i.e.,
\[
\operatorname{rank}_{t}(\boldsymbol{\mathcal{A}})=\#\{i: \boldsymbol{\mathcal{S}}(i, i,:)\neq0\}\:.
\]
\end{mydef}
We will also need the following definition of tensor multi-rank.

\begin{lemma}
{
[Best tubal rank-$r$ approximation \cite{kilmer2013third}] Let the t-SVD of $\boldsymbol{\mathcal{A}} \in \mathbb{R}^{n_1 \times n_2 \times n_3}$ be $\boldsymbol{\mathcal{A}}=\boldsymbol{\mathcal{U}} * \boldsymbol{\mathcal{S}} * \boldsymbol{\mathcal{V}}^{*}$. Given a tubal rank $r$, define $\boldsymbol{\mathcal{A}}_{r}:=\sum_{s=1}^{r} \boldsymbol{\mathcal{U}}(:, s,:) * \boldsymbol{\mathcal{S}}(s, s,:) * \boldsymbol{\mathcal{V}}^{*}(:, s,:)$.
Then 
\[
\boldsymbol{\mathcal{A}}_{r}=\underset{\check{\boldsymbol{\mathcal{A}}} \in \mathbb{A}}{\arg \min }\|\boldsymbol{\mathcal{A}}-\check{\boldsymbol{{\mathcal{A}}}}\|_{F}, 
\]
where $\mathbb{A}:=\left\{\boldsymbol{\mathcal{X}} * \boldsymbol{\mathcal{Y}} \mid \boldsymbol{\mathcal{X}} \in\mathbb{R}^{n_1 \times r \times n_3}, \boldsymbol{\mathcal{Y}} \in \mathbb{R}^{r \times n_2 \times n_3}\right\}$.
\label{lm:approx}
}
\end{lemma}

\begin{mydef}
[Tensor Multi-Rank \cite{kilmer2013third}] For any tensor $\boldsymbol{\mathcal{A}} \in \mathbb{R}^{n_{1} \times n_{2} \times n_{3}},$ its multi-rank rank$_{m}(\boldsymbol{\mathcal{A}})$ is a vector
defined as $\boldsymbol{r}=\left(\operatorname{rank}(\bar{\boldsymbol{A}}^{(1)}) ; \cdots ; \bar{\boldsymbol{A}}^{\left(n_{3}\right)}\right). $ {Specifically, the relation between tubal-rank and tensor multi-rank is
\[
\operatorname{rank}_{t}(\boldsymbol{\mathcal{A}})=\max \left(r_{1}, \cdots, r_{n_{3}}\right)\:,
\]
where $r_i$ is the $i$-th element of $\boldsymbol{r}$.}
\end{mydef}

\subsection{Low-tubal-rank Tensor Completion}
\label{sec:LTR_TC}
Tensor completion is the task of recovering a tensor $\boldsymbol{\cal{M}}\in\mathbb{R}^{n_1\times n_2\times n_3}$ from a subset of its entries by leveraging the low-rank property of the tensor. When using tubal rank for the definition of the rank, the low-tubal-rank property typically amounts to $\operatorname{rank}_{t}(\boldsymbol{\mathcal{M}})\ll \max\{n_1,n_2\}$. Specifically, by defining the observed subset of entries $\boldsymbol{\Omega}\subseteq[n_1]\times[n_2]\times[n_3]$ and its indicator tensor $\boldsymbol{\cal{P}}$,
\begin{equation}
{\cal{P}}_{ijk}=\left\{\begin{array}{cl}
1, & \text {if }(i,j,k) \in \Omega \\
0, & \text {otherwise}
\end{array}\right.
\end{equation}
the low-tubal-rank tensor completion problem can be formulated through the minimization, 
\begin{equation}
\underset{\boldsymbol{\mathcal{Z}} \in \mathbb{R}^{n_1 \times n_2 \times n_3}}{\min }\operatorname{rank}_{t}(\boldsymbol{\mathcal{Z}}) \text {,  s.t. } {\boldsymbol{\cal{P}}}\circ(\boldsymbol{\mathcal{Z}}- \boldsymbol{\mathcal{M}})=\boldsymbol{0}\:,
\label{eq:MCorigin}
\end{equation}
where $\circ$ denotes the Hadamard (element-wise) product of the two same-size tensors.
It is known that \eqref{eq:MCorigin} is NP-hard. To address this problem, several methods were proposed, which can be categorized into two main categories: 

1) Convex relaxation \cite{zhang2016exact,hu2016twist}: In this approach, \eqref{eq:MCorigin} is relaxed to obtain a convex optimization problem. 
Specifically, 
by defining the tensor nuclear norm (TNN) $$\|\boldsymbol{\mathcal{A}}\|_{TNN}=\frac{1}{n_{3}} \sum_{i=1}^{n_{3}}\left\|\bar{\boldsymbol{A}}^{(i)}\right\|_{*}$$ where $\|\cdot\|_{*}$ denotes the matrix nuclear norm, \eqref{eq:MCorigin} can be relaxed to
\begin{equation}
\underset{\boldsymbol{\mathcal{Z}} \in \mathbb{R}^{n_1 \times n_2 \times n_3}}{\min }\sum_{i=1}^{n_{3}}\left\|\bar{\boldsymbol{Z}}^{(i)}\right\|_{*} \text {,  s.t. } {\boldsymbol{\cal{P}}}\circ(\boldsymbol{\mathcal{Z}}- \boldsymbol{\mathcal{M}})=\boldsymbol{0}\:.
\end{equation}
The iterative solver to the nuclear norm-based relaxation has to compute a SVD at each iteration, which incurs high computational complexity for large scale high-dimensional data.

2) Tensor factorization: Similar to the Powerfactorization method proposed for matrix completion \cite{haldar2009rank}, a low-tubal-rank tensor can be represented as the t-product of two smaller tensors \cite{kilmer2011factorization}. Specifically, the recovered tensor $\boldsymbol{\cal{M}}\in\mathbb{R}^{n_1\times n_2 \times n_3}$ can be factorized into the t-product of two tensors $\boldsymbol{\cal{X}}\in\mathbb{R}^{n_1\times r\times n_3}$ and $\boldsymbol{\cal{Y}}\in\mathbb{R}^{r\times n_2 \times n_3}$, where $r$ is the tubal rank of $\boldsymbol{\cal{M}}$ \cite{zhou2017tensor}. The tensor factorization then solves tensor completion by utilizing the objective function
\begin{equation}
\min_{{\boldsymbol{\mathcal{X}}, \boldsymbol{\mathcal{Y}}}} J({\boldsymbol{\mathcal{X}}, \boldsymbol{\mathcal{Y}}}):=\left\|{\boldsymbol{\cal{P}}}\circ(\boldsymbol{\mathcal{X}} * \boldsymbol{\mathcal{Y}}-\boldsymbol{\mathcal{M}})\right\|_{F}^{2}\:.
\label{eq:TF}
\end{equation}
Tensor factorization can avoid the high complexity associated with performing the SVD,  
and the complexity is reduced due to the inherent low-rank property. Two algorithms based on tensor factorization were proposed, namely, Tubal-Altmin (TAM) \cite{liu2019low} and TCTF \cite{zhou2017tensor}.

{Low-tubal-rank-based tensor completion offers several advantages over tensor completion using other tensor rank models (e.g., CP rank, Tucker rank, and tensor ring rank). First, other methods usually impose low rank constraints through the nuclear norm minimization on unfolding matrices of the tensor, which {may destroy} the original multi-dimensional structure of the tensor data. By contrast, based on tensor algebra, the tubal rank based methods directly impose a low-tubal-rank constraint on a tensor, and can well capture the inherent low-rank structure of a tensor \cite{zhou2017tensor,liu2019low}. Second, unlike other rank models for which it is hard or infeasible to obtain an optimal approximation with truncated decomposition, in the tubal rank model, such an approximation is given in Lemma \ref{lm:approx}, which gives a theoretical footing for our proposed method. Third, if the tensor has a large $n_3$, the dimensions of the unfolding matrices will be very large, which degrades the computational efficiency of said algorithms. On the other hand, for the tubal-rank-based method, the SVD or factorization are applied to matrices of size $n_1\times n_2$, which are smaller than the unfolding matrices.} 

\subsection{Maximum Correntropy Criterion (MCC)}
Correntropy is a local and nonlinear similarity measure between two random variables within a ``window'' in the joint space determined by the kernel width.
Given two random variables $X$ and $Y$, the correntropy is defined as \cite{liu2007correntropy}
\begin{equation}
V (X, Y)= \mathbb{E}[\kappa_{\sigma} (X, Y)]=\int \kappa_{\sigma} (x, y)dF_{XY} (x, y)\:,
\end{equation}
where $\kappa_\sigma$ is a shift-invariant Mercer kernel with kernel width $\sigma$, $F _{XY} (x, y)$ denotes the joint probability distribution function of $X$ and $Y$, and $\mathbb{E}[.]$ is the expectation operator.
Given a finite number of samples $ \{x_i, y_i \} _{i=1}^N$, and using the Gaussian kernel, $G_\sigma(x)=\exp (-\frac{x^2}{2\sigma^2})$, as the kernel function, the correntropy can be approximated by
\begin{equation}
\hat{V} (X, Y)= \frac{1}{N} \sum_{i=1}^N \exp (-\frac{e_i^2}{2\sigma^2})\:,
\end{equation}
where $e_i=x_i-y_i$.
\par
Compared with the $\ell_2$-norm based second-order statistic of the error, the correntropy involves all the even moments of the difference between $X$ and $Y$ and is insensitive to outliers. Replacing the second-order measure with the correntropy measure leads to the maximum correntropy criterion (MCC) \cite{Singh2009Using}. The MCC solution is obtained by maximizing the following utility function
\begin{equation}
\label{MCC}
{J_{mcc}} = \mathbb{E}\left[ {{G_{\sigma}\! }\left ( {{e (i)}}\right)}\right]\:.
\end{equation}
Moreover, in practice, the MCC can also be formulated as minimizing the following correntropy-induced loss (C-loss) function \cite{singh2014c}
\begin{equation}
\label{eq:C-loss}
J_{C-loss}= \frac{1}{M}\sum_{i=1}^M\sigma^2\left(1-{{G_{\sigma}\! }\left ( {{e (i)}} \right)}\right)\:.
\end{equation}
The cost function above is closely related to Welsch's cost function, originally introduced in \cite{dennis1978techniques}.

{Fig.~\ref{fig:correntropy} illustrates the different error measures. As can be seen, the correntropy-based error measure can efficiently reduce the effect of a large error $e$ caused by large outliers. Compared with $\ell_1$-norm-based error, it is also differentiable at $0$, which is convenient for optimization and allows us to leverage an HQ technique to efficiently solve the problem. The superior performance of correntropy over $\ell_1$ and $\ell_2$ norm has been verified in many fields \cite{he2019robust,zheng2020broad}, and is also verified in this work.}

\begin{figure}[ht]
\centering
\includegraphics[width=0.5\textwidth]{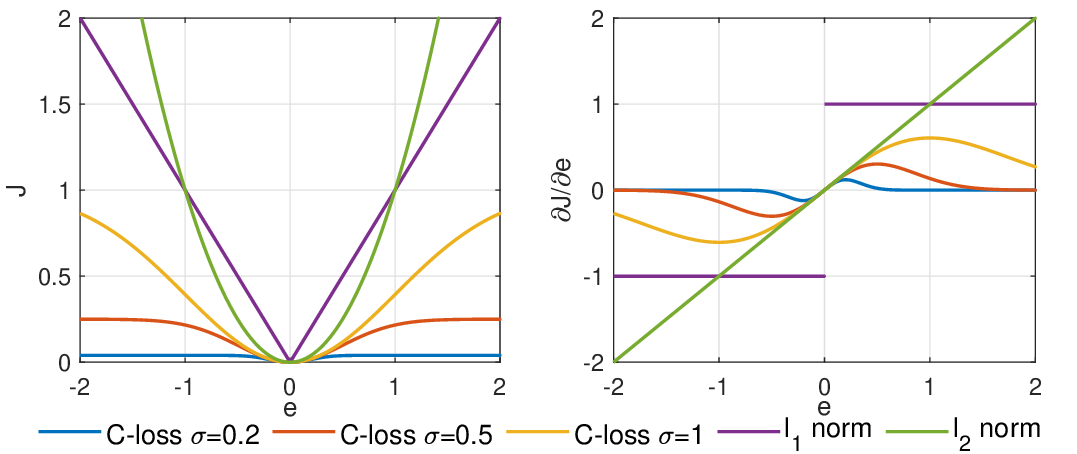}
\caption{Curves of different error measures with error $e$. Left: Cost function $J$ with $e$. Right: Derivative $\partial J/\partial e$ with respect to $e$.}
\label{fig:correntropy}
\end{figure}

\section{Proposed Methods}
\label{sec:methods}
\subsection{Correntropy-based Tensor Completion}

Before we state our objective function for tensor completion, we first rewrite \eqref{eq:TF} as
\begin{equation}
\min_{{\boldsymbol{\mathcal{X}}, \boldsymbol{\mathcal{Y}}}} J({\boldsymbol{\mathcal{X}}, \boldsymbol{\mathcal{Y}}}):=\sum\limits_{i = 1}^{n_1} \sum\limits_{j = 1}^{n_2} \sum\limits_{k = 1}^{n_3}{{{{\cal{P}} _{ijk}}{{\left( {{{\cal{M}}_{ijk}} - {\left( {{\boldsymbol{\cal{X}}*\boldsymbol{{\cal{Y}}} }}\right)_{ijk}}} \right)}^2}} }\:.
\end{equation}
When the observed entries ${\cal{M}}_{ijk}$ are corrupted or contain large outliers, the $\ell_2$ error measure can bias the optimization, which degrades the performance of tensor completion. To enhance robustness, in this work we utilize the correntropy as the error measure. By replacing the $\ell_2$ error measure with correntropy, we obtain the new optimization problem
\begin{equation}
\label{eq:GTF}
\begin{aligned}
&\min_{{\boldsymbol{\mathcal{X}}, \boldsymbol{\mathcal{Y}}}} J_{G_{\sigma}}({\boldsymbol{\mathcal{X}}, \boldsymbol{\mathcal{Y}}})\\
&:=\!\sum\limits_{i = 1}^{n_1} \sum\limits_{j = 1}^{n_2} \sum\limits_{k = 1}^{n_3} \!{{\cal{P}}_{ijk}\sigma^2\left({1-G_{\sigma}{{\left( {{{\cal{M}}_{ijk}} - {\left( {{\boldsymbol{\cal{X}}*\boldsymbol{{\cal{Y}}} }}\right)_{ijk}}} \right)}}} \right)}
\end{aligned}
\end{equation}
The formulation in \eqref{eq:GTF} 
generalizes the correntropy-based formulation in \cite{he2019robust} for matrix completion. In particular, for the special case where $n_3=1$, the optimization in \eqref{eq:GTF} reduces to the correntropy-based matrix completion. Surely, since tensor algebra is substantially different from the algebra of matrices (even the definition of tensor rank is not unique), the solution in \cite{he2019robust} is no longer suitable for tensor completion, a fact which will also be verified in Section IV. Thus, here we seek new approaches to solve \eqref{eq:GTF}.

\subsection{Optimization via Half-quadratic Minimization}
\label{sec:HC}
In general, \eqref{eq:GTF} is non-convex and is difficult to be directly optimized. To tackle this difficulty, we utilize the half-quadratic (HQ) optimization technique to optimize the correntropy-based cost function. According to the half-quadratic optimization theory \cite{nikolova2005analysis}, 
there exists a convex conjugated function $\varphi$ such that
\begin{equation}
\label{eq:HQ1}
G _\sigma (e)=\max_{t}\left(\frac{e^2t}{\sigma^2}-\varphi(t)\right)\:,
\end{equation}
where $t\in\mathbb{R}$ and the maximum is reached at $t=-G _\sigma (e)$. Eq. \eqref{eq:HQ1} can be rewritten as
\begin{equation}
\sigma^2(1-G _\sigma (e))=\min_{t}\left(-e^2t+\sigma^2\varphi(t)\right)\:.
\end{equation}
By defining $s=-t$ and $\phi(s)=\sigma^2\varphi(-s)$, the above equation can be written as
\begin{equation}
\label{eq:HQ2}
\min_{e}\sigma^2(1-G _\sigma (e))=\min_{e,s}\left(e^2s+\phi(s)\right)\:.
\end{equation}
Thus, minimizing the non-convex C-loss function in terms of $e$ is equivalent to minimizing an augmented cost function in an enlarged parameter space $\{e,s\}$.
Therefore, by substituting \eqref{eq:HQ2} in \eqref{eq:GTF}, the correntropy-based objective function $J_{G_{\sigma}}({\boldsymbol{\cal{X}},\boldsymbol{\cal{Y}}})$ can be expressed as
\begin{align}
\begin{aligned}
J_{G_{\sigma}}\!({\boldsymbol{\mathcal{X}}, \boldsymbol{\mathcal{Y}}})\!=\mathop {\min}\limits_{\boldsymbol{\cal{W}}} \sum\limits_{i = 1}^{n_1}\! \sum\limits_{j = 1}^{n_2} \!\sum\limits_{k = 1}^{n_3}\! &
\Bigg({{{{\cal{W}}_{ijk}}{{\cal{P}} _{ijk}}{{\left({\cal{M}}_{ijk}\! -\! \left( \boldsymbol{{\cal{X}}}\!*\!\boldsymbol{{\cal{Y}}} \right)_{ijk} \right)}^2}} }
\\
& 
+ {{\cal{P}} _{ijk}}  \phi \left( {{{\cal{W}}_{ijk}}} \right)\Bigg)
\end{aligned}
\end{align}
Further, by defining the augmented cost function
\begin{equation}
\label{eq:HQ3}
J_{HQ}({\boldsymbol{\cal{X}},\boldsymbol{\cal{Y}}},\boldsymbol{\cal{W}})\!=\!{\|\sqrt{\boldsymbol{\cal{W}}} \circ \boldsymbol{\cal{P}} \circ \left(\boldsymbol{\cal{M}}\!-\!\boldsymbol{\cal{X}}*\boldsymbol{\cal{Y}}\right)\|_{F}^2}+\psi_{\boldsymbol{\Omega}} \left( {\boldsymbol{\cal{W}}}\right)
\end{equation}
where $\psi_{\boldsymbol{\Omega}} \left( {\boldsymbol{\cal{W}}}\right)=\sum\nolimits_{i = 1}^{n_1} \sum\nolimits_{j = 1}^{n_2} \sum\nolimits_{k = 1}^{n_3}{{\cal{P}} _{ijk}}\phi \left( {{{\cal{W}}_{ijk}}} \right)$, we have the following relation
\begin{equation}
\label{HQG}
\min_{\boldsymbol{\cal{X}},\boldsymbol{\cal{Y}}}J_{G_{\sigma}}({\boldsymbol{\cal{X}},\boldsymbol{\cal{Y}}})=\min_{\boldsymbol{\cal{X}},\boldsymbol{\cal{Y}},\boldsymbol{\cal{W}}}J_{HQ}({\boldsymbol{\cal{X}},\boldsymbol{\cal{Y}},\boldsymbol{\cal{W}}})\:.
\end{equation}
Therefore, the correntropy-based optimization problem is formulated as a half-quadratic based optimization.

We propose the following alternating minimization procedure to solve the optimization problem \eqref{eq:HQ3}:

\subsubsection{Optimizing $\boldsymbol{\cal{W}}$}

According to \eqref{eq:HQ1} and \eqref{eq:HQ2}, given a certain $e$, the minimum is reached at $s=G_{\sigma}(e)$. Therefore, given the fixed $\boldsymbol{\cal{X}}$ and $\boldsymbol{\cal{Y}}$, the optimal solutions of ${\cal{W}}_{ijk}$ for $(i,j,k)\in \boldsymbol{\Omega}$ can be obtained as
\begin{equation}
\label{eq:HQW}
{\cal{W}}_{ijk}=G_{\sigma}{{\left( {{{\cal{M}}_{ijk}} - {\left( {{\boldsymbol{\cal{X}}*\boldsymbol{{\cal{Y}}} }}\right)_{ijk}}} \right)}}, (i,j,k)\in\boldsymbol{\Omega}\:.
\end{equation}
Since computing ${\cal{W}}_{ijk}$ for $(i,j,k)\notin\boldsymbol{\Omega}$ does not affect the solution of \eqref{eq:GTF} due to the multiplication with $\boldsymbol{\cal{P}}$, henceforth 
we use ${\cal{W}}_{ijk}$ for all the entries to simplify the expressions.

\subsubsection{Optimizing $\boldsymbol{\cal{X}}$ and $\boldsymbol{\cal{Y}}$}

Given a fixed $\boldsymbol{\cal{W}}$, \eqref{eq:HQ3} becomes a weighted tensor completion problem
\begin{equation}
\label{eq:WTF}
\min_{\boldsymbol{{\cal{X}}},\boldsymbol{\cal{Y}}}{\|\sqrt{\boldsymbol{\cal{W}}} \circ \boldsymbol{{\cal{P}}} \circ \left(\boldsymbol{\cal{M}}-\boldsymbol{\cal{X}*\cal{Y}}\right)\|_{F}^2}\:.
\end{equation}
The weighting tensor $\boldsymbol{\cal{W}}$ assigns different weights to each observed entry based on error residuals. Given the nature of the Gaussian function, a large error will lead to a small weight, such that the negative impact of large outliers for error statistics can be greatly reduced. In the following, we propose and develop two algorithms to solve \eqref{eq:WTF}.

\subsection{Alternating Minimization-based Algorithm}
\label{sec:HQ-TCTF}
Inspired by TCTF \cite{zhou2017tensor}, we first propose an alternating minimization-based approach to solve \eqref{eq:WTF}. By introducing an auxiliary tensor variable $\boldsymbol{\cal{Z}}$, \eqref{eq:WTF} can be rewritten as
\begin{equation}
\begin{aligned}
\label{eq:TCTF_J}
\min_{\boldsymbol{{\cal{X}}},\boldsymbol{\cal{Y}},\boldsymbol{\cal{Z}}} J({\boldsymbol{{\cal{X}}},\boldsymbol{\cal{Y}},\boldsymbol{\cal{Z}}}):=&\|\sqrt{\boldsymbol{\cal{W}}} \circ \boldsymbol{\cal{P}} \circ \left(\boldsymbol{\cal{M}}-\boldsymbol{\cal{Z}}\right)\|_{F}^2\\
&+\beta\|\boldsymbol{\cal{X}*\cal{Y}}-\boldsymbol{\cal{Z}}\|_{F}^2\:.
\end{aligned}
\end{equation}
where $\beta$ is the regularization parameter. To solve \eqref{eq:TCTF_J}, one can again utilize alternating minimization and update $\boldsymbol{\cal{Z}}$, $\boldsymbol{\cal{X}}$ and $\boldsymbol{\cal{Y}}$ in turn. Specifically, by fixing $\boldsymbol{\cal{X}}$ and $\boldsymbol{\cal{Y}}$, we can update $\boldsymbol{\cal{Z}}$ as
\begin{equation}
\label{eq:TCTF_Z}
\boldsymbol{\cal{Z}}=\arg\min_{\boldsymbol{\cal{Z}}}{\|\sqrt{\boldsymbol{\cal{W}}} \circ \boldsymbol{\cal{P}} \circ \left(\boldsymbol{\cal{M}}-\boldsymbol{\cal{Z}}\right)\|_{F}^2+\beta\|\boldsymbol{\cal{X}*\cal{Y}}-\boldsymbol{\cal{Z}}\|_{F}^2}
\end{equation}
To solve \eqref{eq:TCTF_Z}, we set the first derivative of  $J({\boldsymbol{{\cal{X}}},\boldsymbol{\cal{Y}},\boldsymbol{\cal{Z}}})$ with respect to $\boldsymbol{\cal{Z}}$ to zero, i.e.,
\begin{equation}
\label{eq:TCTF_pZ}
\frac{\partial J}{\partial {\boldsymbol{\cal{Z}}}}=2\left( \boldsymbol{\cal{W}} \circ \boldsymbol{\cal{P}} \circ \left(\boldsymbol{\cal{Z}}-\boldsymbol{\cal{M}}\right)+\beta\boldsymbol{\cal{Z}}-\beta\boldsymbol{\cal{X}}*\boldsymbol{\cal{Y}}\right)=\boldsymbol{0}
\end{equation}
Eq. \eqref{eq:TCTF_pZ} is equivalent to the requirement that
\begin{equation}
\left.\begin{array}{cl}
\boldsymbol{\cal{P}} \circ \left(\boldsymbol{\cal{W}} \circ\boldsymbol{\cal{Z}}-\boldsymbol{\cal{W}} \circ\boldsymbol{\cal{M}}+\beta\boldsymbol{\cal{Z}}-\beta\boldsymbol{\cal{X}}*\boldsymbol{\cal{Y}}\right)=\boldsymbol{0} \\
(\boldsymbol{{1}}-\boldsymbol{\cal{P}})\circ(\boldsymbol{\cal{Z}}-\boldsymbol{\cal{X}}*\boldsymbol{\cal{Y}})=\boldsymbol{0}
\end{array}\right.
\end{equation}
Thus, $\boldsymbol{\cal{Z}}$ can be obtained in closed-form as
\begin{equation}
\begin{aligned}
\boldsymbol{\cal{Z}}&=\boldsymbol{\cal{P}}\circ\boldsymbol{\cal{Z}}+(\boldsymbol{{1}}-\boldsymbol{\cal{P}})\circ\boldsymbol{\cal{Z}}\\
&=\boldsymbol{\cal{X}}*\boldsymbol{\cal{Y}}+\frac{\boldsymbol{\cal{W}}}{\beta\boldsymbol{1}+\boldsymbol{\cal{W}}}\circ\boldsymbol{\cal{P}}\circ(\boldsymbol{\cal{M}}-\boldsymbol{\cal{X}}*\boldsymbol{\cal{Y}})
\end{aligned}
\label{eq:TCTF_upZ}
\end{equation}
where $\boldsymbol{1}$ denotes the tensor of all ones, and the division is element-wise. 
Further, by fixing $\boldsymbol{\cal{Z}}$, \eqref{eq:TCTF_J} reduces to the following minimization: 
\begin{equation}
\min_{\boldsymbol{{\cal{X}}},\boldsymbol{\cal{Y}}}{\|\boldsymbol{\cal{X}*\cal{Y}}-\boldsymbol{\cal{Z}}\|_{F}^2}\:.
\end{equation}
According to Lemma \ref{lem:t-product}, we have
\begin{equation}
{\|\boldsymbol{\cal{X}*\cal{Y}}-\boldsymbol{\cal{Z}}\|_{F}^2}=\frac{1}{n_3}{\|\boldsymbol{\bar{X}}\boldsymbol{\bar{Y}}-\boldsymbol{\bar{Z}}\|_{F}^2}\:.
\end{equation}
Given the block structure of $\boldsymbol{\bar{X}}$, $\boldsymbol{\bar{Y}}$ and $\boldsymbol{\bar{Z}}$, the above minimization problem is equivalent to solving the $n_3$ subproblems
\begin{equation}
\min_{\boldsymbol{\bar{X}}^{(k)},\boldsymbol{\bar{Y}}^{(k)}}{\|\boldsymbol{\bar{X}}^{(k)}\boldsymbol{\bar{Y}}^{(k)}-\boldsymbol{\bar{Z}}^{(k)}\|_{F}^2},k=1,\ldots,n_3\:.
\end{equation}
For each $k$, we can alternate between least-squares solutions to $\boldsymbol{\bar{X}}^{(k)}$ and $\boldsymbol{\bar{Y}}^{(k)}$, i.e., 
\begin{equation}
\begin{aligned}
\label{eq:TCTF_XY}
\boldsymbol{\bar{X}}^{(k)}&={\boldsymbol{\bar{Z}}}^{(k)}\left(\bar{\boldsymbol{Y}}^{(k)}\right)^{*}\left(\bar{\boldsymbol{Y}}^{(k)}\left(\bar{\boldsymbol{Y}}^{(k)}\right)^{*}\right)^{\dagger}\\
\boldsymbol{\bar{Y}}^{(k)}&=\left(\bar{\boldsymbol{Y}}^{(k)}\left(\bar{\boldsymbol{Y}}^{(k)}\right)^{*}\right)^{\dagger}\left(\bar{\boldsymbol{Y}}^{(k)}\right)^{*}{\boldsymbol{\bar{Z}}}^{(k)}
\end{aligned}
\end{equation}
where $\boldsymbol{A}^{\dagger}$ denotes the Moore-Penrose pseudo-inverse of matrix $\boldsymbol{A}$.  Therefore, to solve \eqref{eq:HQ3}, we alternate between the updates in \eqref{eq:HQW}, \eqref{eq:TCTF_upZ} and \eqref{eq:TCTF_XY} until convergence. We name this algorithm `Half-Quadratic based Tensor Completion by Tensor Factorization' (HQ-TCTF). The  pseudocode of HQ-TCTF is summarized in Algorithm 1. Note that in step 3 of the algorithm we use an adaptive kernel width to enhance the rate of convergence. More details about this strategy are discussed in Section \ref{sec:stopping}.

Note that the $n_3$ subproblems in each alternating minimization step are independent of each other. Thus, the solution to these subproblems can be parallelized to further speed up computation. 

\begin{algorithm}
\caption{HQ-TCTF for robust tensor completion}
\begin{algorithmic}[1]
 \REQUIRE $\boldsymbol{\cal{P}}$, $\boldsymbol{\cal{P}}\circ\boldsymbol{\cal{M}}$, $\beta$ and $r$
 \STATE initial tensors $\boldsymbol{\cal{X}}^0$ and $\boldsymbol{\cal{Y}}^0$, $t=0$\\
 \REPEAT
 \STATE compute $\sigma^{t+1}$ and $\boldsymbol{\cal{W}}^{t+1}$ using \eqref{eq:HQW}.
 \STATE compute $\boldsymbol{\cal{Z}}^{t+1}$ using \eqref{eq:TCTF_Z}.
 \FOR {$k=1,...,n_3$}
 \STATE compute $\boldsymbol{\cal{X}}^{(k),t+1}$ and $\boldsymbol{\cal{Y}}^{(k),t+1}$ using \eqref{eq:TCTF_XY}
 \ENDFOR
 \STATE $t=t+1$
 \UNTIL stopping criterion is satisfied

\ENSURE $\boldsymbol{\boldsymbol{\cal{X}}}^{t}*{\boldsymbol{\cal{Y}}}^{t}$
\end{algorithmic}
\label{alg:HQ-TCTF}
\end{algorithm}

\begin{remark}
One can observe that as $\sigma\rightarrow\infty$, $G_{\sigma}(e)$ approaches $1$, thus all the entries of $\boldsymbol{\cal{W}}$ become $1$. In this special case, one does not need to optimize $\boldsymbol{\cal{W}}$ in \eqref{eq:HQW}, and \eqref{eq:HQ3} reduces to
\begin{equation}
\min_{\boldsymbol{{\cal{X}}},\boldsymbol{\cal{Y}}}\|{\boldsymbol{{\cal{P}}} \circ \left(\boldsymbol{\cal{M}}-\boldsymbol{\cal{X}*\cal{Y}}\right)\|_{F}^2}
\:,
\end{equation}
which is the tensor completion problem in \eqref{eq:TF}. Further, by setting $\beta=0$ in \eqref{eq:TCTF_upZ}, the updates of $\boldsymbol{\cal{Z}}$, $\boldsymbol{\cal{X}}$ and $\boldsymbol{\cal{Y}}$ will be the same as in TCTF.
\end{remark}

\begin{remark}
The adaptive tubal rank estimation method developed for TCTF \cite{zhou2017tensor} can be naturally applied to HQ-TCTF. Specifically, the scalar rank parameter $r$ in Algorithm 1 can be replaced with a multi-rank vector $\boldsymbol{r}=[r_1,\ldots,r_{n_3}]$
and the adaptive approach in \cite{xu2015parallel,zhou2017tensor} iteratively estimates the rank of the tensor.  
\end{remark}

The following proposition establishes convergence guarantees for HQ-TCTF. 
\begin{proposition} Define the cost function
\begin{equation}
\begin{aligned}
J({\boldsymbol{\cal{X}},\boldsymbol{\cal{Y}}},\boldsymbol{\cal{Z}},\boldsymbol{\cal{W}})=&{\|\sqrt{\boldsymbol{\cal{W}}} \circ \boldsymbol{\cal{P}} \circ \left(\boldsymbol{\cal{M}}\!-\!\boldsymbol{\cal{X}}*\boldsymbol{\cal{Y}}\right)\|_{F}^2}\\
&+\!\|\boldsymbol{\cal{X}*\cal{Y}}\!-\!\boldsymbol{\cal{Z}}\|_{F}^2+\psi_{\boldsymbol{\Omega}} \left( {\boldsymbol{\cal{W}}}\right)\:.
\end{aligned}
\end{equation}

The sequence $\{J_{\sigma^t}({\boldsymbol{\cal{X}}^t,\boldsymbol{\cal{Y}}}^t,\boldsymbol{\cal{Z}}^t,\boldsymbol{\cal{W}}^t),t=1,2,\ldots\}$ generated by Algorithm \ref{alg:HQ-TCTF} converges.
\end{proposition}

\begin{proof} 
Since $\boldsymbol{\cal{W}}$ and $\boldsymbol{\cal{Z}}$ are optimal solutions to \eqref{eq:HQW} and \eqref{eq:TCTF_Z}, respectively, we have
\begin{equation}
\label{eq:TCTF_p1_J}
J({\boldsymbol{\cal{X}}^{t+1},\boldsymbol{\cal{Y}}}^{t+1},\boldsymbol{\cal{Z}}^{t+1},\boldsymbol{\cal{W}}^{t+1})\leq J({\boldsymbol{\cal{X}}^{t+1},\boldsymbol{\cal{Y}}}^{t+1},\boldsymbol{\cal{Z}}^{t},\boldsymbol{\cal{W}}^{t}) \:.
\end{equation}
Then, from Lemma 3 in the supplementary material of \cite{zhou2017tensor}, one can obtain that for each matrix
${\boldsymbol{\bar{X}}^{(k)},\boldsymbol{\bar{Y}}^{(k)}},k=1,...,n_3$ generated from \eqref{eq:TCTF_XY}, the following inequality holds
\begin{equation}
{\|\boldsymbol{\bar{X}}^{(k),t+1}\boldsymbol{\bar{Y}}^{(k),t+1}-\boldsymbol{\bar{Z}}\|_{F}^2}\leq{\|\boldsymbol{\bar{X}}^{(k),t}\boldsymbol{\bar{Y}}^{(k),t}-\boldsymbol{\bar{Z}}\|_{F}^2}
\end{equation}
From Lemma \ref{lem:t-product}, we have ${\|\boldsymbol{\cal{X}}^{t}*\boldsymbol{\cal{Y}}^{t}-\boldsymbol{\cal{Z}}\|_{F}^2}=\frac{1}{n_3}\sum_{k=1}^{n_3}{\|\boldsymbol{\bar{X}}^{(k),t}\boldsymbol{\bar{Y}}^{(k),t}-\boldsymbol{\bar{Z}}\|_{F}^2}$. Thus the following inequality holds
\begin{equation}
\label{eq:TCTF_p1_XY}
{\|\boldsymbol{\cal{X}}^{t+1}*\boldsymbol{\cal{Y}}^{t+1}-\boldsymbol{\cal{Z}}\|_{F}^2}\leq{\|\boldsymbol{\cal{X}}^{t}*\boldsymbol{\cal{Y}}^{t}-\boldsymbol{\cal{Z}}\|_{F}^2}
\end{equation}
Combining \eqref{eq:TCTF_p1_J} and \eqref{eq:TCTF_p1_XY} we have
\begin{equation}
J({\boldsymbol{\cal{X}}^{t+1},\boldsymbol{\cal{Y}}}^{t+1},\boldsymbol{\cal{Z}}^{t+1},\boldsymbol{\cal{W}}^{t+1})\leq J({\boldsymbol{\cal{X}}^{t},\boldsymbol{\cal{Y}}}^{t},\boldsymbol{\cal{Z}}^{t},\boldsymbol{\cal{W}}^{t})
\end{equation}
It can be also verified that $J({\boldsymbol{\cal{X}}^{t},\boldsymbol{\cal{Y}}}^{t},\boldsymbol{\cal{Z}}^{t},\boldsymbol{\cal{W}}^{t}) $ is always bounded below for arbitrary $t$. Thus, $\{J({\boldsymbol{\cal{X}}^{t},\boldsymbol{\cal{Y}}}^{t},\boldsymbol{\cal{Z}}^{t},\boldsymbol{\cal{W}}^{t}),t=1,2,...\}$ will converge.
\end{proof}

\subsection{Alternating Steepest Descent-based Algorithm}
\label{sec:HQ-TCASD}
In the context of matrix completion, alternating steepest descent (ASD) was introduced to efficiently solve the completion problem \cite{tanner2016low}. ASD has a lower per-iteration complexity than PowerFactorization, and can recover high rank matrices. In this section, we introduce the ASD method for tensor completion and develop an efficient robust tensor completion algorithm.

As mentioned in Section \ref{sec:HC}, we first optimize $\boldsymbol{\cal{W}}$ using \eqref{eq:HQW}. Then, instead of directly optimizing \eqref{eq:WTF}, we gradually update $\boldsymbol{\cal{X}}$ and $\boldsymbol{\cal{Y}}$ using gradient descent. For convenience, we first add a multiplicative factor of $\frac{1}{2}$ to \eqref{eq:WTF} such that the minimization problem becomes
\begin{equation}
\frac{1}{2}\min_{\boldsymbol{{\cal{X}}},\boldsymbol{\cal{Y}}}{\|\sqrt{\boldsymbol{\cal{W}}} \circ \boldsymbol{\cal{P}} \circ \left(\boldsymbol{\cal{M}}-\boldsymbol{\cal{X}*\cal{Y}}\right)\|_{F}^2}\:.
\label{eq:ASD_J}
\end{equation}
Then, using the relation \eqref{eq:prelim_1} and Definition \ref{def:tprod} in Section \ref{sec:prelim}, \eqref{eq:ASD_J} can be rewritten as
\begin{equation}
\frac{1}{2}\min_{\boldsymbol{{\cal{X}}},\boldsymbol{\cal{Y}}}{\|\sqrt{\tilde{\boldsymbol{{W}}}} \circ \tilde{\boldsymbol{{P}}} \circ \left(\tilde{\boldsymbol{{M}}}-\operatorname{bcirc}({\boldsymbol{\cal{X}})\tilde{\boldsymbol{{Y}}}}\right)\|_{F}^2}
\label{eq:ASD_J2}\:.
\end{equation}
Based on the block-circulant diagonalization \cite{kilmer2013third}, we have
\begin{equation}
\begin{aligned}
\operatorname{bcirc}({\boldsymbol{\cal{X}})\tilde{\boldsymbol{{Y}}}}&=\left(\boldsymbol{F}_{n_{3}}^{-1} \otimes \boldsymbol{I}_{n_{1}}\right)  \bar{\boldsymbol{X}}{\hat{\boldsymbol{{ Y }}}}\\
&=\boldsymbol{F}^{-1}\bar{\boldsymbol{X}}{\hat{\boldsymbol{{ Y }}}}\\
&=\boldsymbol{U}{\hat{\boldsymbol{{ Y }}}}
\end{aligned}
\end{equation}
where $\boldsymbol{F}^{-1}=\boldsymbol{F}_{n_{3}}^{-1} \otimes \boldsymbol{I}_{n_{1}}$ (consequently $\boldsymbol{F}=\boldsymbol{F}^{-1} \times n_3$), $\boldsymbol{U}=\boldsymbol{F}^{-1}\bar{\boldsymbol{X}}$ and  ${\hat{\boldsymbol{A}}}=\operatorname{unfold}(\bar{\boldsymbol{\cal{A}}})$. Finally, \eqref{eq:ASD_J2} can be reformulated as
\begin{equation}
\min J({\boldsymbol{U},{{\hat{\boldsymbol{Y}}}}}):=\frac{1}{2}{\left\|{\sqrt{\tilde{\boldsymbol{W}}} \circ \tilde{\boldsymbol{P}} \circ \left({\tilde{\boldsymbol{M}}}-{\boldsymbol{U}{\hat{{\boldsymbol{Y}}}}}\right)}\right\|_{F}^2}\:.
\end{equation}
Using the matrix derivatives, the partial derivative of $J({\boldsymbol{U},\hat{\boldsymbol{Y}}})$ with respect to $\boldsymbol{U}$ can be computed as
\begin{equation}
\boldsymbol{g}_{\boldsymbol{U}}=\frac{\partial J}{\partial {\boldsymbol{U} }}=-{\tilde{\boldsymbol{W}}} \circ \tilde{\boldsymbol{P}} \circ  \left(\tilde{{\boldsymbol{M}}}-{\boldsymbol{U}{{\hat{\boldsymbol{Y}}}}}\right){{\hat{\boldsymbol{Y}}}}^* \:.
\end{equation}

Note that ${\bar{\boldsymbol{X}}}=\boldsymbol{F}\boldsymbol{U}$ is a block diagonal matrix. Following the method in  \cite{gilman2020grassmannian}, we force the update of ${\bar{\boldsymbol{X}}}$ at each iteration to be block diagonal. Specifically, by defining the operator $\operatorname{bdiagz}(\cdot)$ which sets the non-block-diagonal entries of a matrix to zero, the updated gradient can be obtained as
\begin{equation}
\boldsymbol{g}'_{\boldsymbol{U}}=\boldsymbol{F}^{-1}\operatorname{bdiagz}(\boldsymbol{F}\boldsymbol{g}_{\boldsymbol{U}})\:.
\label{eq:ASD_gU}
\end{equation}
The steepest descent step size  $\mu'_{\boldsymbol{U}}$ for $\boldsymbol{g}'_{\boldsymbol{U}}$ can be obtained as the following minimizer
\begin{equation}
\begin{aligned}
\mu'_{{{\boldsymbol{U}}}}&=\arg\min_{\mu}{\left\|{\sqrt{\tilde{\boldsymbol{W}}} \circ \tilde{\boldsymbol{P}} \circ \left({\tilde{\boldsymbol{M}}}-({\boldsymbol{U}-\mu\boldsymbol{g}'_{\boldsymbol{U}}){\hat{{\boldsymbol{Y}}}}}\right)}\right\|_{F}^2}\\
&=\frac{\|\boldsymbol{g}'_{{{\boldsymbol{U}}}}\|_F^2}{\left\|\sqrt{\tilde{\boldsymbol{W}}} \circ \tilde{\boldsymbol{P}} \circ\left( {\boldsymbol{g}}'_{{\boldsymbol{U}}}{{\hat{\boldsymbol{Y}}}}\right)\right\|_F^2}
\end{aligned}
\end{equation}
and the matrix ${\boldsymbol{U}}$ can be updated as
\begin{equation}
\label{eq:ASD_U}
\boldsymbol{U}^{t+1}=\boldsymbol{U}^{t}-\mu_{\boldsymbol{U}}'^t{\boldsymbol{g}}'^t_{\boldsymbol{U}}\:.
\end{equation}
Similarly, by fixing $\boldsymbol{U}$, the partial derivative of $\boldsymbol{J}$ w.r.t. $\hat{\boldsymbol{Y}}$ can be obtained as
\begin{equation}
\boldsymbol{g}_{\hat{\boldsymbol{Y}}}=\frac{\partial J}{\partial {\hat{\boldsymbol{Y}}}}=-{\boldsymbol{U}}^*\left({\tilde{\boldsymbol{W}}} \circ \tilde{\boldsymbol{P}} \circ  \left(\tilde{{\boldsymbol{M}}}-{\boldsymbol{U}{{\hat{\boldsymbol{Y}}}}}\right)\right)\:.
\label{eq:ASD_gY}
\end{equation}
The corresponding step size $\mu_{{\hat{\boldsymbol{Y}}}}$ will be
\begin{equation}
\begin{aligned}
{\mu}_{{\hat{\boldsymbol{Y}}}}&=\frac{\|\boldsymbol{g}_{{\hat{\boldsymbol{Y}}}}\|_F^2}{\left\|\sqrt{\tilde{\boldsymbol{W}}} \circ \tilde{\boldsymbol{P}} \circ \left({\boldsymbol{U}}{\boldsymbol{g}}_{{\hat{\boldsymbol{Y}}}}\right)\right\|_F^2}\:.
\end{aligned}
\end{equation}

Similar to ASD, the foregoing update process suffers from a slow rate of  convergence when directly applied to image and video completion tasks. To tackle this problem, following a Newton-like method for Scaled ASD \cite{tanner2016low}, we scale the gradient descent direction for $\hat{\boldsymbol{Y}}$ in \eqref{eq:ASD_gY} by $(\boldsymbol{U}^*\boldsymbol{U})^{-1}$, i.e.,
\begin{equation}
\boldsymbol{g}'_{\hat{\boldsymbol{Y}}}=(\boldsymbol{U}^*\boldsymbol{U})^{-1}\boldsymbol{g}_{\hat{\boldsymbol{Y}}}\:,
\label{eq:ASD_gY2}
\end{equation}
and the corresponding step size $\mu'_{{\hat{\boldsymbol{Y}}}}$ with exact line-search is
\begin{equation}
\begin{aligned}
{\mu}'_{{\hat{\boldsymbol{Y}}}}&=\frac{\langle\boldsymbol{g}_{\hat{\boldsymbol{Y}}},\boldsymbol{g}'_{\hat{\boldsymbol{Y}}}\rangle}{\left\|\sqrt{\tilde{\boldsymbol{W}}} \circ \tilde{\boldsymbol{P}} \circ \left({\boldsymbol{U}}{\boldsymbol{g}}'_{{\hat{\boldsymbol{Y}}}}\right)\right\|_F^2}\:,
\end{aligned}
\end{equation}
where $\langle \boldsymbol{A}, \boldsymbol{B}\rangle := \sum_{1 \leq i, j \leq n} A^*_{ij} B_{ij}$. Therefore, the matrix $\hat{\boldsymbol{Y}}$ at the $t$-th iteration can be updated by combining \eqref{eq:ASD_gY} and \eqref{eq:ASD_gY2}, i.e.,
\begin{equation}
\label{eq:ASD_Y}
{\hat{\boldsymbol{Y}}}^{t+1}={\hat{\boldsymbol{Y}}}^{t}-(1-\lambda)\mu_{{\hat{\boldsymbol{Y}}}}^t\boldsymbol{g}_{{\hat{\boldsymbol{Y}}}}^t-\lambda\mu'^t_{{\hat{\boldsymbol{Y}}}}\boldsymbol{g}'^t_{{\hat{\boldsymbol{Y}}}}\:,
\end{equation}
where $0\leq\lambda\leq1$ is a free parameter to be chosen.

Therefore, the matrices ${\boldsymbol{U}}$ and ${{\hat{\boldsymbol{Y}}}}$ can be alternately updated using \eqref{eq:ASD_U} and \eqref{eq:ASD_Y} until convergence. We term the above algorithm `Half-Quadratic based Tensor Completion by Alternating Steepest Descent' (HQ-TCASD).

Similar to HQ-TCTF, adaptive selection of the kernel width $\sigma$ is used to improve the rate of convergence and the performance of HQ-TCASD. HQ-TCASD is summarized in Algorithm \ref{alg:HQ-TCASD}. We remark that the matrices ${\boldsymbol{U}}(\bar{\boldsymbol{X}})$ and ${{\hat{\boldsymbol{Y}}}}$ have a block structure, so the matrix computation can be processed block-by-block. Also, since we have $\boldsymbol{F}\tilde{\boldsymbol{A}}=\operatorname{unfold(fft}(\boldsymbol{\cal{A}},[\:],3))$ for a tensor $\boldsymbol{\cal{A}}$, the conventional FFT operation can be used in \eqref{eq:ASD_gU} instead of matrix multiplication to further accelerate the computation.

\begin{algorithm}
\caption{HQ-TCASD for robust tensor completion}
\begin{algorithmic}[1]
 \REQUIRE $\boldsymbol{\cal{P}}$, $\boldsymbol{\cal{P}}\circ\boldsymbol{\cal{M}}$, $r$ and $\lambda$
 \STATE initial matrices $\boldsymbol{U}^0$ and $\hat{\boldsymbol{{Y}}}^0$, $t=0$\\
 \REPEAT
 \STATE compute $\sigma^{t+1}$ and $\boldsymbol{\cal{W}}^{t+1}$ using \eqref{eq:HQW}.
 \STATE compute $\boldsymbol{U}^{t+1}$ using \eqref{eq:ASD_U}.
 \STATE compute $\hat{\boldsymbol{Y}}^{t+1}$ using \eqref{eq:ASD_Y}.
 \STATE $t=t+1$
 \UNTIL stopping criterion is satisfied

\ENSURE $\boldsymbol{U}^{t}*{\hat{\boldsymbol{Y}}}^{t}$
\end{algorithmic}
\label{alg:HQ-TCASD}
\end{algorithm}

The following proposition verifies the convergence of the proposed HQ-TCSAD.

\begin{proposition}
Define the cost function
\begin{equation}
\begin{aligned}
J({\boldsymbol{\cal{X}},\boldsymbol{\cal{Y}}},\boldsymbol{\cal{W}})=&\frac{1}{2}{\|\sqrt{\boldsymbol{\cal{W}}} \circ \boldsymbol{\cal{P}} \circ \left(\boldsymbol{\cal{M}}\!-\!\boldsymbol{\cal{X}}*\boldsymbol{\cal{Y}}\right)\|_{F}^2}\\
&+\frac{1}{2}\psi_{\boldsymbol{\Omega}} \left( {\boldsymbol{\cal{W}}}\right)
\end{aligned}
\end{equation}
The sequence $\{J({\boldsymbol{\cal{X}}^t,\boldsymbol{\cal{Y}}}^t,\boldsymbol{\cal{W}}^t),t=1,2,...\}$ generated by Algorithm 2 will converge.
\end{proposition}
\begin{proof}
See Appendix A.
\end{proof}

\begin{remark}
As $\sigma\rightarrow\infty$ (i.e., for the standard tensor completion cost function in \eqref{eq:TF}), one can set $\tilde{\boldsymbol{W}}$ to be the all ones matrix and alternately update \eqref{eq:ASD_U} and \eqref{eq:ASD_Y}. This is itself a new algorithm, which we term TCASD. It can be used for tensor completion in noise-free settings or with Gaussian noise.
\end{remark}
\subsection{Stopping Criterion and Adaptive Kernel Width Selection}
\label{sec:stopping}
The relative error between iterations can be used to measure the speed of convergence and develop a stopping criterion. Specifically, the residual error tensor at the $t$-th iteration $\boldsymbol{\cal{E}}^t$ is defined as
\begin{equation}
\boldsymbol{\mathcal{E}}^t=\sqrt{\boldsymbol{\cal{W}}^t}\circ\boldsymbol{\cal{P}}\circ(\boldsymbol{\mathcal{M}}-\boldsymbol{\mathcal{X}^t}*\boldsymbol{\mathcal{Y}}^t)\:.
\end{equation}
If $\|\boldsymbol{\cal{E}}^{t}\|_F-\|\boldsymbol{\cal{E}}^{t-1}\|_F$ falls below a sufficiently small value $\varepsilon$, the algorithm is considered to have converged to a local minimum, and the iterative procedure terminates.

To further improve performance and achieve a faster rate of convergence, we use an adaptive kernel width selection strategy. Specifically, the kernel width at the $t+1$-th iteration is determined by
\begin{equation}
\label{sigmaout}
\sigma^{t+1}=\max\left(\eta\left(\max({\boldsymbol{e}_{\Omega}^t}_{(0.25)},{\boldsymbol{e}_{\Omega}^t}_{(0.75)})\right),\sigma_{min}\right)
\end{equation}
where $\boldsymbol{e}^{t}_{\Omega}\in\mathbb{R}^{|\Omega|\times 1}$ denotes the vector composed of all non-zero entries of $\boldsymbol{\cal{P}}\circ(\boldsymbol{\mathcal{M}}-\boldsymbol{\mathcal{X}^t}*\boldsymbol{\mathcal{Y}}^t)$, and $\boldsymbol{y}_{(q)}$ denotes the $q$-th quantile of $\boldsymbol{y}$. The parameter $\eta$ controls the kernel width, and $\sigma_{min}$ is a lower bound on $\sigma$.

\subsection{Complexity Analysis}
\label{sec:complexity}
We first present a complexity analysis of HQ-TCTF. Computing $\sigma$ involves computing $\boldsymbol{\cal{X}}*\boldsymbol{\cal{Y}}$ and finding the quantile of $e_{\Omega}$, whose time complexities are ${\cal{O}}(r(n_1+n_2)n_3\log n_3+rn_1n_2n_3)$ and ${\cal{O}}(n_1n_2n_3)$, respectively. The complexity of computing $\boldsymbol{\cal{W}}$ and $\boldsymbol{\cal{Z}}$ are both ${\cal{O}}(n_1n_2n_3)$ since $\boldsymbol{\cal{X}}*\boldsymbol{\cal{Y}}$ was already computed. Then, the cost of updating $\boldsymbol{\cal{X}}$ and $\boldsymbol{\cal{Y}}$ is ${\cal{O}}(r(n_1+n_2)n_3\log n_3+rn_1n_2n_3)$. Therefore, the overall complexity of HQ-TCTF is ${\cal{O}}(r(n_1+n_2)n_3\log n_3+rn_1n_2n_3)$.

For HQ-TCASD, similar to HQ-TCTF, the complexity of computing $\sigma$ is ${\cal{O}}(r(n_1+n_2)n_3\log n_3+rn_1n_2n_3)$. Computing $\boldsymbol{g}'_{\boldsymbol{U}}$ using FFT has complexity ${\cal{O}}(r(n_1+n_2)n_3\log n_3+rn_1n_3\max(n_2,n_3))$, and calculation of $\mu'_{\boldsymbol{U}},\boldsymbol{g}_{\hat{\boldsymbol{Y}}}$ and $\mu_{\hat{\boldsymbol{Y}}}$ is of complexity ${\cal{O}}(r(n_1+n_2)n_3\log n_3+rn_1n_2n_3)$. Therefore, the overall complexity of HQ-TCASD is ${\cal{O}}(r(n_1+n_2)n_3\log n_3+rn_1n_3\max(n_2,n_3))$. 

One can observe that if $n_2 > n_3$, both HQ-TCTF and HQ-TCASD have the same order complexity. {Further, as both HQ-TCTF and HQ-TCASD are SVD-free algorithms and are readily parallelizable, the computation 
can be easily accelerated through parallel computation, which is verified in the experiments.}

\section{Experiments}
\label{sec:results}
In this section, we thoroughly evaluate the performance of the proposed algorithms HQ-TCTF, HQ-TCASD  and TCASD using both synthetic and real data. We compare to existing tensor completion algorithms, including TCTF \cite{zhou2017tensor}, TAM \cite{liu2019low} and TNN \cite{zhang2016exact}, and robust tensor completion algorithms,  including SNN-L1 \cite{goldfarb2014robust}, SNN-HT/ST with Welsch loss (SNN-WHT/WST) \cite{yang2015robust}, TRNN-L1 \cite{huang2020robust} and TNN-L1 \cite{jiang2019robust}. For fair comparison, the adaptive kernel width selection method is also applied to SNN-WHT and SNN-WST in the experiments. Further, the correntropy-based robust matrix completion algorithm \cite{he2019robust} is also included in the comparisons, where the tensor is treated as $n_3$ matrices of dimension $n_1\times n_2$. In the experiments, we refer to this matrix-completion-based method as HQ-MCASD. {We also report the run-time of the proposed methods on a GPU (designated with suffix 'parallel') by simply using the `gpuArray' data structure in MATLAB. All algorithms are implemented using MATLAB r2019b on a standard 16-GB memory PC with a 2.6-GHz CPU and an NVIDIA RTX3070 GPU.}

\begin{figure*}[htbp]
\centering
\includegraphics[width=0.9\textwidth]{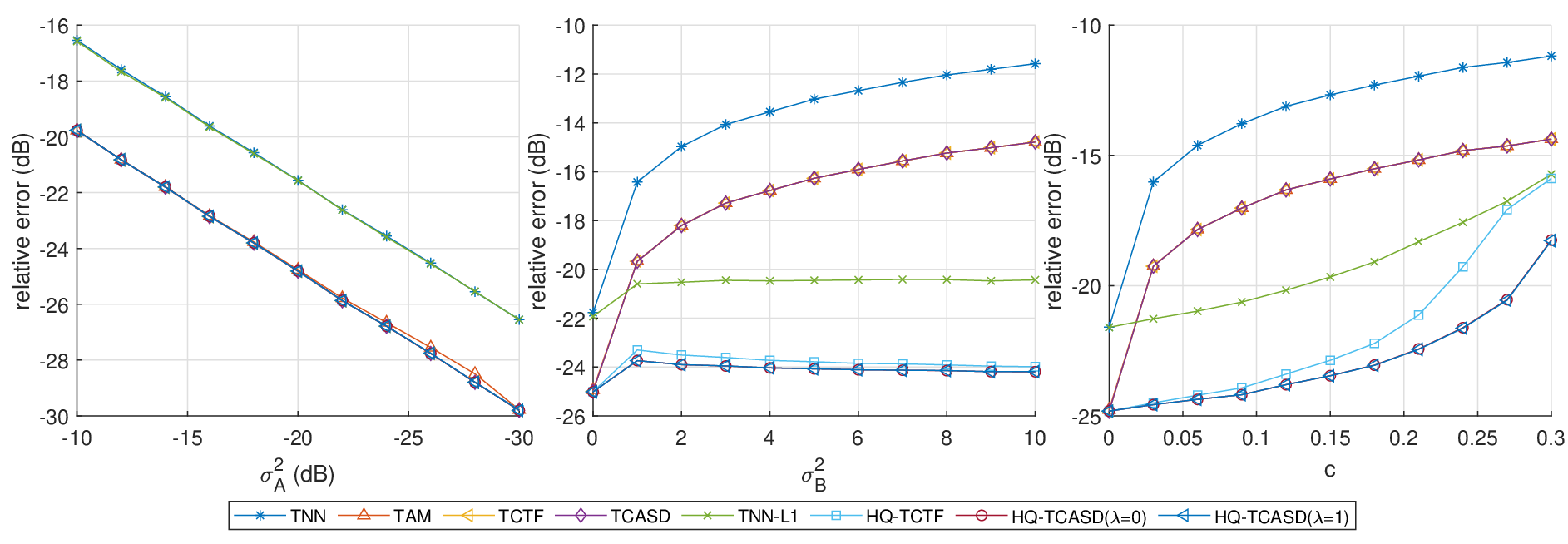} 
\caption{Curves of average relative error under different noise environments. Left: $c=0$. Middle: $\sigma_A^2=0.01,c=0.1$. Right: $\sigma_A^2=0.01$,$\sigma_B^2=1$.}
\label{fig:syn_GMM}
\end{figure*}

\begin{figure}[htbp]
\centering
\includegraphics[width=0.45\textwidth]{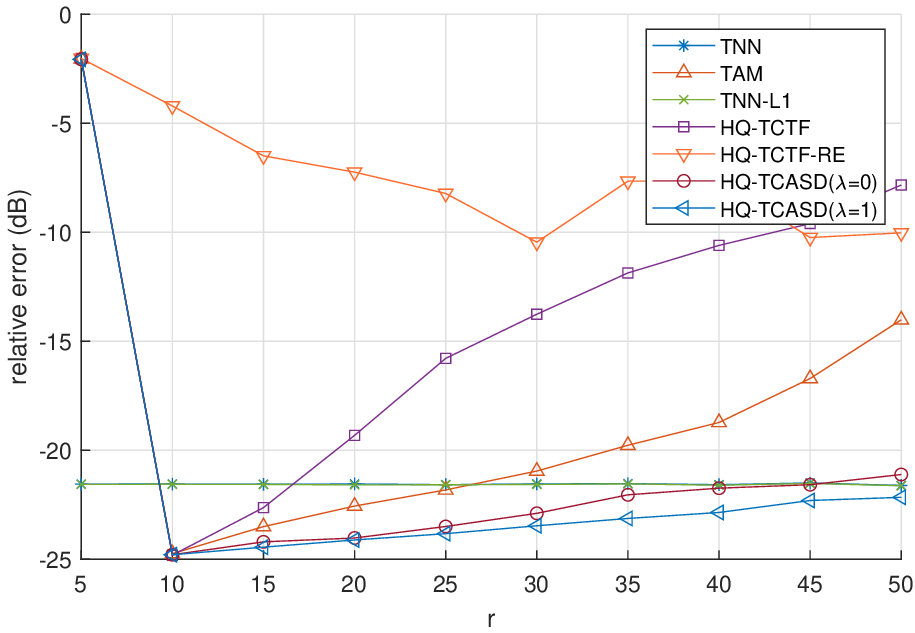} 
\caption{Average relative error as function of the rank parameter $r$ with Gaussian noise.}
\label{fig:syn_rank}
\end{figure}

\begin{figure}[htbp]
\centering
\includegraphics[width=0.48\textwidth]{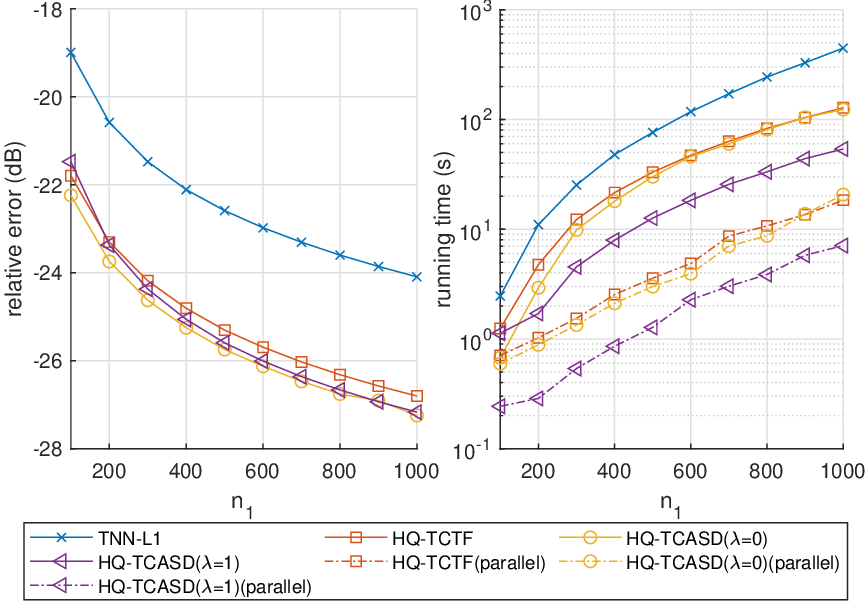} 
\caption{{Average relative error (left) and average run-time (right) as function of $n_1$ under GMM noise.}}
\label{fig:syn_size}
\end{figure}

In all simulations, the maximum number of iterations of all algorithms is set to $500$ unless explicitly mentioned. The parameter $\eta$ in \eqref{sigmaout} for adaptive kernel width selection is set to $6$ and $2$ for synthetic data and real data, respectively. The lower bound $\sigma_{min}$ for kernel width selection is experimentally set to $0.3$ for synthetic data and $0.15$ for real data. The threshold $\varepsilon$ for the stopping criterion is set to $10^{-9}$ for synthetic data and $10^{-5}$ for real data. The regularization parameter $\beta$ for HQ-TCTF is set to $1$. For real data, the $\lambda$ for HQ-TCASD is fixed to $0.2$. Other parameters for each algorithm are tuned to achieve the best performance in each task. Note that the parameters of the different algorithms are not adapted across different noise settings in each simulation. Fixing the parameters is important since, the noise properties could be changing and may not be measurable in practice.

\subsection{Synthetic Data}
In this section, we verify the performance of the proposed algorithms using synthetic data. The dimensions of the tensor are set to $n_1=n_2=200, n_3=20$. The low-tubal-rank tensor $\boldsymbol{\cal{M}}$ with tubal rank $\bar{r}$ is obtained by the t-product of two tensors whose entries are generated from a zero mean Gaussian distribution with unit variance. The indicator tensor $\boldsymbol{\cal{P}}$ with observation fraction $p$ is generated by randomly and uniformly assigning $p\times 100\%$ of the entries of $\boldsymbol{\cal{P}}$ the value $1$. The performance of an instance of tensor completion is evaluated using the relative error
\begin{equation}
\label{NMSE}
rel.err = \frac{{{{\left\|{\hat {\boldsymbol{\cal{M}}}}-\boldsymbol{\cal{M}} \right\|}_F}}}{{{\| \boldsymbol{\cal{M}} \|}_F}}\:,
\end{equation}
where $\hat{\boldsymbol{\cal{M}}}$ is the recovered tensor. The performance is evaluated by taking the ensemble average of the relative error over $T$ independent Monte Carlo runs of different instances of $\boldsymbol{\cal{P}}$ and the noise. In this section, we only compare the performance of the proposed algorithms to TNN, TNN-L1, TAM and TCTF since the other algorithms are using different definitions for the tensor rank.

In the experiments, the observed entries of the tensor are perturbed by additive noise generated from the standard two-component Gaussian mixture model (GMM). The probability density function is given by $(1-c)N(0,{\sigma_A^2})+cN(0,\sigma_B^2)$, where $N(0,{\sigma_A^2})$ represents the general Gaussian noise disturbance with variance ${\sigma_A^2}$, and $N(0,\sigma_B^2)$ with a large variance $\sigma_B^2$ captures the outliers. The variable $c$ controls the occurrence probability of outliers.

\begin{table*}[]
{
\caption{Completion Performance (PSNR) Comparison on four images from the DAVIS and SIDD dataset}
\label{tb:image}
\resizebox{\textwidth}{!}{%
\begin{tabular}{@{}ccccccccccccc@{}}
\toprule
Image                     & \begin{tabular}[c]{@{}c@{}}Missing\\ Pattern\end{tabular} & Noise                   & c    & \begin{tabular}[c]{@{}c@{}}Image\\ SNR\end{tabular} & SNN-L1         & \begin{tabular}[c]{@{}c@{}}SNN-\\ WST\end{tabular} & \begin{tabular}[c]{@{}c@{}}SNN-\\ WHT\end{tabular} & TRNN-L1 & TNN-L1 & \begin{tabular}[c]{@{}c@{}}HQ-\\ MCASD\end{tabular} & \begin{tabular}[c]{@{}c@{}}HQ-\\ TCTF\end{tabular} & \begin{tabular}[c]{@{}c@{}}HQ-\\ TCASD\end{tabular} \\ \midrule
\multirow{4}{*}{bus}      & \multirow{4}{*}{random(50\%)}& \multirow{4}{*}{\begin{tabular}[c]{@{}c@{}}Stripe\\ GMM\end{tabular}} 
                                                                  & 0    & 23.62 & 25.36          & 26.27& 25.61& 24.58   & 27.25  & 26.26 & {\ul 28.76} & \textbf{29.12} \\
                          &             &                         & 0.1  & 7.69 & 25.01          & 25.69& 25.19& 23.95   & 26.25  & 25.75 & {\ul 28.16} & \textbf{28.45}                                      \\
                          &             &                         & 0.2  & 3.03 & 24.13          & 24.85& 23.84& 23.21   & 24.98  & 20.64 & {\ul 26.47} & \textbf{27.26}                                      \\
                          &             &                         & 0.3  & -1.34 & 23.10          & {\ul 23.66} & 19.97& 22.46   & 23.22  & 11.95 & 20.45& \textbf{24.50}                                      \\ \midrule
\multirow{4}{*}{dance}    & \multirow{4}{*}{watermark}   & \multirow{4}{*}{\begin{tabular}[c]{@{}c@{}}Stripe\\ PSP\end{tabular}} 
                                                                  & 0    & 22.44 & 29.06          & {\ul 29.96} & 29.47& 29.85   & 29.25  & 29.59 & 27.47& \textbf{31.24}  \\
                          &             &                         & 0.1  & 6.01 & 27.07          & 21.63& 28.91& 28.36   & 27.12  & {\ul 28.94}  & 27.22& \textbf{30.88}                                      \\
                          &             &                         & 0.2  & 3.21 & 24.05          & 13.24& {\ul 27.99} & 25.94   & 23.27  & 25.41 & 26.20& \textbf{29.49}                                      \\
                          &             &                         & 0.3  & 1.38 & 17.69          & 11.28& \textbf{24.53}        & 23.10   & 15.36  & 11.94 & 18.47& {\ul 24.38}                               \\ \midrule
\multirow{4}{*}{board}    & \multirow{4}{*}{random(50\%)}& \multirow{4}{*}{Unknown}& \multirow{4}{*}{\textbackslash{}} 
                                                                         & 24.49 & 37.91          & 38.59& 34.44& 37.94   & 38.23  & 37.23 & {\ul 39.43} & \textbf{39.58}  \\
                          &             &                         &      & 18.88 & 34.12          & 37.68& 33.89& 37.23   & 37.40  & 36.77 & {\ul 38.55} & \textbf{38.66}                                      \\
                          &             &                         &      & 13.90 & \textbf{30.89} & {\ul 30.77} & 28.32& 29.73   & 29.74  & 28.41 & 30.50& 30.46                                               \\
                          &             &                         &      & 9.20 & \textbf{24.98} & 23.42& 21.43& 23.35   & 23.32  & 21.79 & {\ul 23.94} & 23.75                                               \\ \midrule
\multirow{4}{*}{alphabet} & \multirow{4}{*}{watermark}   & \multirow{4}{*}{Unknown}& \multirow{4}{*}{\textbackslash{}} 
                                                                         & 22.63 & 37.50          & \textbf{38.37}        & 36.91& 37.23   & 37.88  & 35.65 & 36.76& {\ul 38.01}  \\
                          &             &                         &      & 19.21 & 31.65          & 33.02& 33.70& 30.59   & 31.78  & 33.20 & {\ul 34.28} & \textbf{34.70}                                      \\
                          &             &                         &      & 16.39 & 28.40          & 29.23& 30.72& 27.29   & 28.48  & 30.73 & {\ul 31.84} & \textbf{31.92}                                      \\
                          &             &                         &      & 12.98 & 25.04          & 25.46& 27.00& 24.06   & 25.09  & 27.42 & {\ul 28.54} & \textbf{28.43}                                      \\ \bottomrule
\end{tabular}%
}
}
\end{table*}

We first investigate the performance of the algorithms under different settings for the noise. The observation fraction $p$ is set to $0.5$ and the tubal rank $\bar{r}$ of $\boldsymbol{\cal{M}}$ is set to $10$. The rank parameter for all the algorithms is set to the true value, i.e. $r=10$. For each noise distribution, we average over $20$ Monte Carlo runs. The average relative error under different noise distributions is shown in Fig.~\ref{fig:syn_GMM}. One can observe that for Gaussian noise (i.e., $c=0$), all algorithms expect TNN and TNN-L1 achieve the same favorable performance, however, for GMM noise with $c\neq 0$, the proposed robust algorithms HQ-TCTF and HQ-TCASD outperform all the other algorithms. Also, HQ-TCASD is shown to slightly outperform HQ-TCTF.

In many practical situations, the actual rank $\bar{r}$ may not be known. Therefore, we study the performance under different settings of $r$. Again, the observation fraction $p$ is set to $0.5$ and the actual tubal rank $\bar{r}= 10$. We use a a Gaussian noise distribution with $\sigma_A^2=0.01$. For all factorization-based algorithms, 
we gradually change the rank parameter $r$ between $5$ and $50$. Note that TNN and TNN-L1 do not require setting the rank since they use convex relaxation as described in Section \ref{sec:LTR_TC}. The other parameters are set as in the previous simulation. For HQ-TCTF, an additional algorithm with adaptive rank estimation (namely HQ-TCTF-RE) is also included for comparison. The average relative error under different rank parameters $r$ is shown in Fig.~\ref{fig:syn_rank} for the different algorithms. As shown, HQ-TCASD is still able to successfully complete the tensor $\boldsymbol{\cal{M}}$ with low relative error even when $r$ is set larger than actual $\bar{r}$.

Finally, we compare the performance of the proposed algorithms and TNN-L1 with different tensor size under the GMM noise model. Here, we only compare to TNN-L1 since it is the only algorithm other than the proposed methods that can yield successful recovery under the GMM noise as shown in Fig. \ref{fig:syn_GMM}. The tensor size is set to $n_1 = n_2$ and $n_3=20$. The parameters of the GMM noise are set to $c=0.1, \sigma_A^2=0.01$ and $\sigma_B^2=10$. The rank $\bar{r}$ is set to $n_1\times0.05$. The rank of HQ-TCASD with $\lambda=1$ is set to $\bar{r}+5$ for fast completion speed. We gradually increase $n_1$ from $100$ to $1000$ and average the relative error over 20 Monte Calro runs. The average relative error and average running time are shown in Fig.\ref{fig:syn_size}. One can observe that the proposed algorithms always yield significant lower relative error and smaller computation time than TNN-L1. {Specifically, the parallel computation can further speed up the computation of the proposed methods by an order of magnitude.}

\begin{figure}[ht]
\centering
\includegraphics[width=0.45\textwidth]{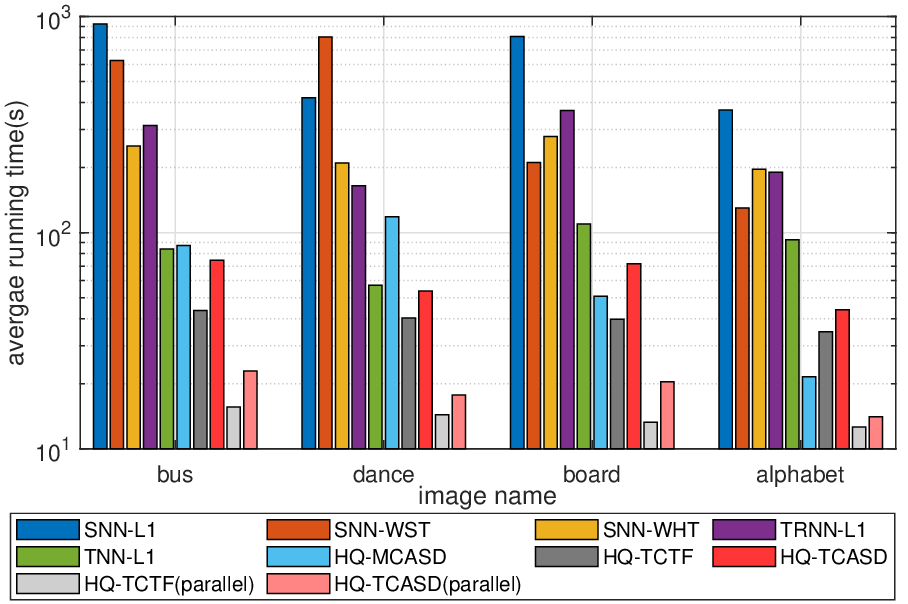}
\caption{Average running time for each image.}
\label{fig:image_time}
\end{figure}

\begin{figure*}[ht]
\centering
\includegraphics[width=1\textwidth]{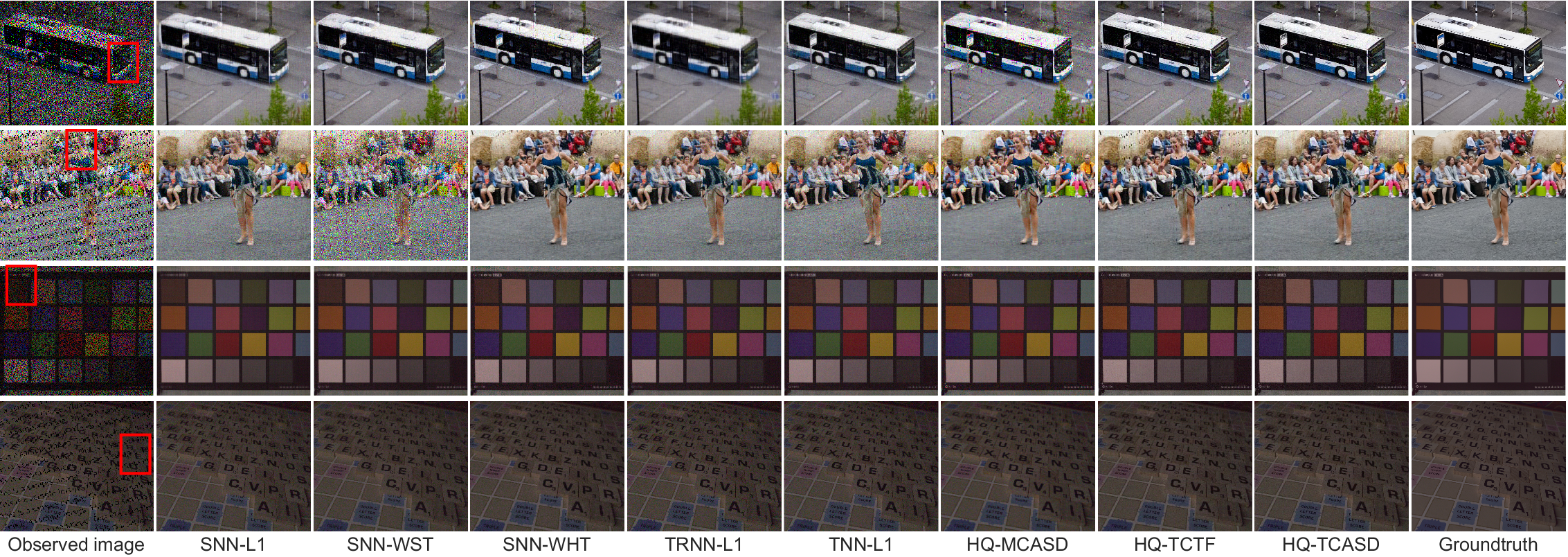}
\caption{{Images recovered by different algorithms under different noise distributions with $c=0.2$.}}
\label{fig:image_large}
\end{figure*}

\begin{figure*}[ht]
\centering
\includegraphics[width=1\textwidth]{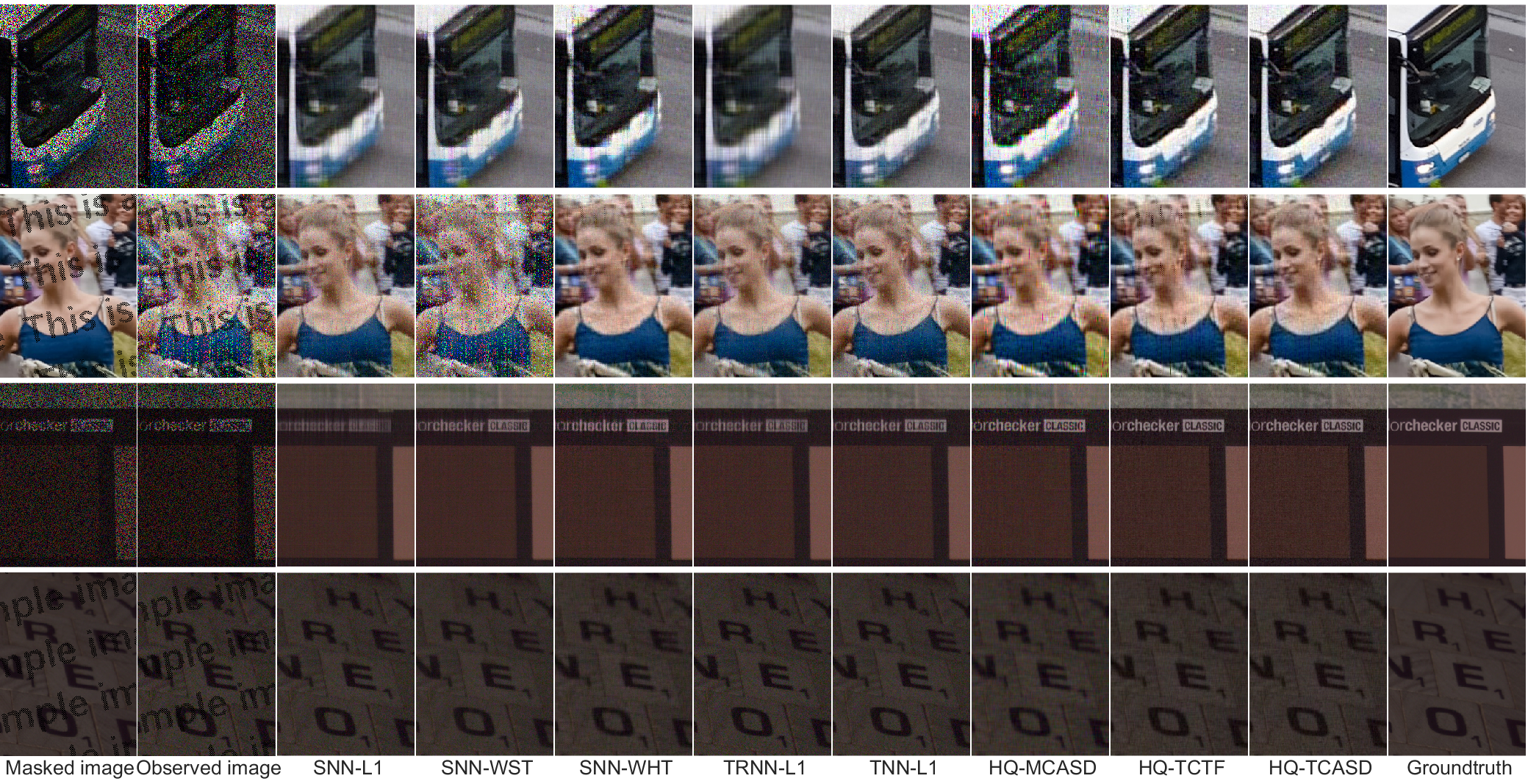}
\caption{{Enlarged regions (red rectangles of Fig.\ref{fig:image_large}) from the images recovered by different algorithms.}}
\label{fig:image_small}
\end{figure*}

\subsection{Image Inpainting}
Image inpainting aims to recover the missing pixels of an image from the observed pixels of the image. Because many images can be well approximated by a low-rank representation, image inpainting can be seen as a matrix or tensor completion task \cite{zhou2017tensor}, and has been widely used for evaluating performance of matrix and tensor completion algorithms. When the observed pixels are corrupted with impulsive noise or outliers, the image inpainting task is more challenging. In this section, we evaluate the performance of the proposed HQ-TCTF and HQ-TCASD algorithms, along with other state-of-the-art robust completion algorithms, on the robust color image inpainting task with multiple noise distributions and missing pixel patterns. The performance evaluation metric is the peak signal-to-noise ratio (PSNR) defined as
\[
PSNR=10\log_{10}\frac{I_{max}^3n_1n_2n_3}{{{\left\| {{\hat {\boldsymbol{\mathcal{M}}}}}-\boldsymbol{\mathcal{M} }\right\|}_F^2}}\:,
\]
where $I_{max}$ denotes the largest value of the pixels of the image data. A higher PSNR signifies better recovery performance.

We evaluate the completion performance of the different algorithms using four images. The first two images `bus' and `paragliding' are chosen from the Densely Annotated Video Segmentation (DAVIS) 2016 dataset\footnote{https://davischallenge.org/davis2016/code.html} \cite{perazzi2016benchmark}. {Different kinds of synthetic noise are added to these two images to obtain the noisy images. The last two images `board' and `alphabet' are selected from the Smartphone Image Denoising Dataset (SIDD)\footnote{https://www.eecs.yorku.ca/$\sim$kamel/sidd/index.php}. For each image, four (noisy) photos captured using a Samsung Galaxy S6 Edge are provided with different lighting conditions along with the ground truth (noiseless) image. The noise comes from the camera itself and no synthetic noise is added.} All images are scaled to $1920\times1080$, so each color image can be regarded as a $1920\times1080\times3$ tensor.

{The completion performance is tested on two types of missing pixel patterns. In the first pattern, we independently and randomly select $50\%$ pixels from each channel as the missing pixels. In the second, a watermark is added to all channels of the image, and the missing pixels correspond to pixels covered by the watermark.} 

{We evaluate performance using four types of noises: 1) \emph{GMM noise}: All observed pixels are perturbed by GMM noise described in the previous experiment with $\sigma_A^2=0.001$ and $\sigma_B^2=1$ and parameter $c$. 2) \emph{Possion+Salt-and-pepper (PSP) noise}: $c\times100\%$ of the observed pixels are randomly selected and perturbed with Salt-and-pepper noise, and the remaining observed pixels are perturbed with Poisson noise. 3) \emph{Stripe GMM noise}: For each channel, $50\%$ of the columns are randomly selected, of which $2c\times100\%$ of the observed pixels are perturbed with Gaussian noise $N(0,1)$. The remaining observed pixels are perturbed by Gaussian noise $N(0,0.01)$. 4) \emph{Stripe PSP noise}: This is similar to Stripe GMM noise, but we replace the Gaussian noise $N(0,1)$ and $N(0,0.01)$ in Stripe GMM noise with Salt-and-pepper noise and Poisson noise, respectively.}

The multi-rank vectors for HQ-TCASD and HQ-TCTF are set to $[150,20,20]$ and $[120,20,20]$, respectively. The average PSNR for the four images is reported in Table \ref{tb:image} for different values of the noise parameter $c$, and the average run-time for each image is shown in Fig. \ref{fig:image_time}. One can observe that HQ-TCASD achieves the highest PSNR for most of the images, and HQ-TCTF is the second best. {Further, parallel computation significantly reduces the computational cost of the proposed HQ-TCTF and HQ-TCASD.} Examples of the recovered full and partially enlarged images are shown in Fig.~\ref{fig:image_large} and Fig.~\ref{fig:image_small}, respectively. As shown, the methods proposed yield visually clearer texture and more accurate colors than the other methods.

\begin{figure*}[htbp]
\centering
\includegraphics[width=1\textwidth]{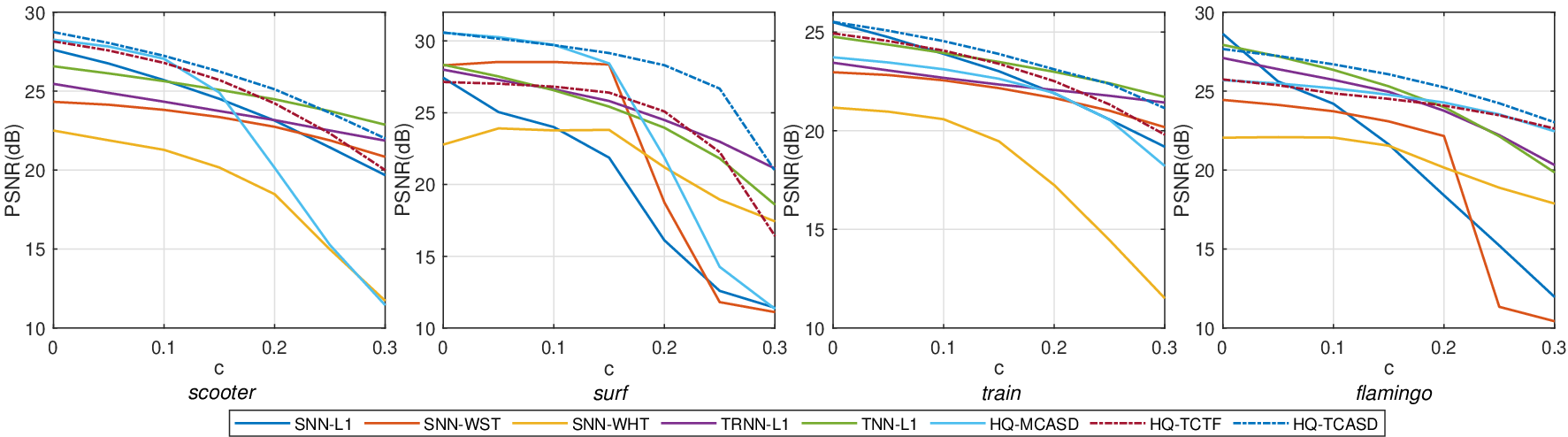} 
\caption{{Average PSNR on four videos from the DAVIS dataset versus parameter $c$. Missing patterns (from left to right): random(50$\%$), watermark, random(50$\%$), watermark. Noise distributions (from left to right): Stripe GMM, Stripe PSP, GMM, PSP. The dashed lines are for the proposed algorithms}.}
\label{fig:video_curves}
\end{figure*}

\subsection{Video Data Completion}
In this section, we evaluate the performance of the algorithms using video data. Four gray-scale video sequences from the DAVIS 2016 dataset are used for testing completion performance. Due to the limitation of the computer memory, the resolution of the video is scaled down to $1280\times 720$ from the original $1920\times 1080$ resolution, and the first 30 frames of each sequence are selected, such that each video sequence forms a tensor of size $1280 \times 720 \times 30$. The multi-rank vectors for HQ-TCASD  and HQ-TCTF are set to $[80, 80,\ldots, 80]$ and $[80, 60,\ldots, 60]$, respectively.

\begin{figure}[htbp]
\centering
\includegraphics[width=0.45\textwidth]{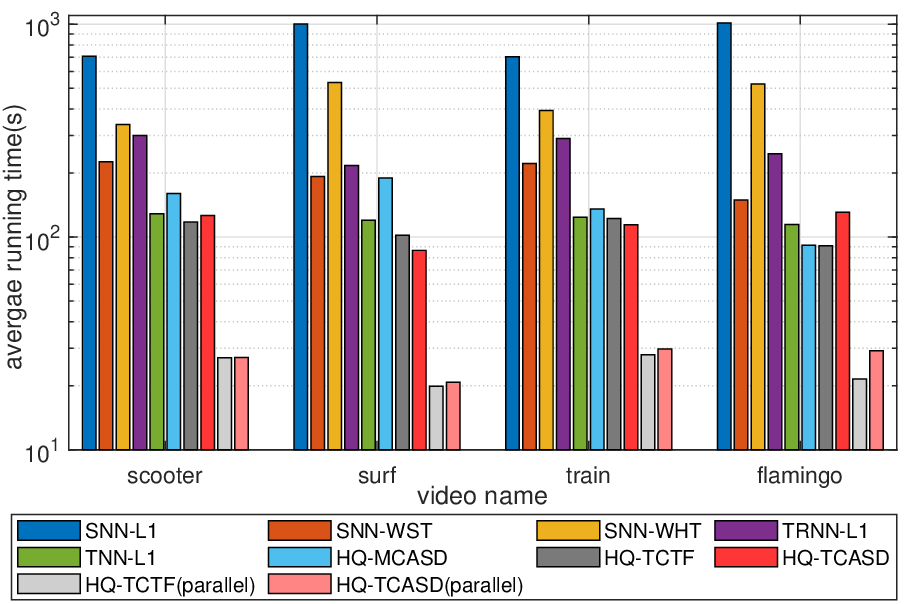} 
\caption{Average running time for each video.}
\label{fig:video_time}
\end{figure}

\begin{figure*}[htbp]
\centering
\includegraphics[width=1\textwidth]{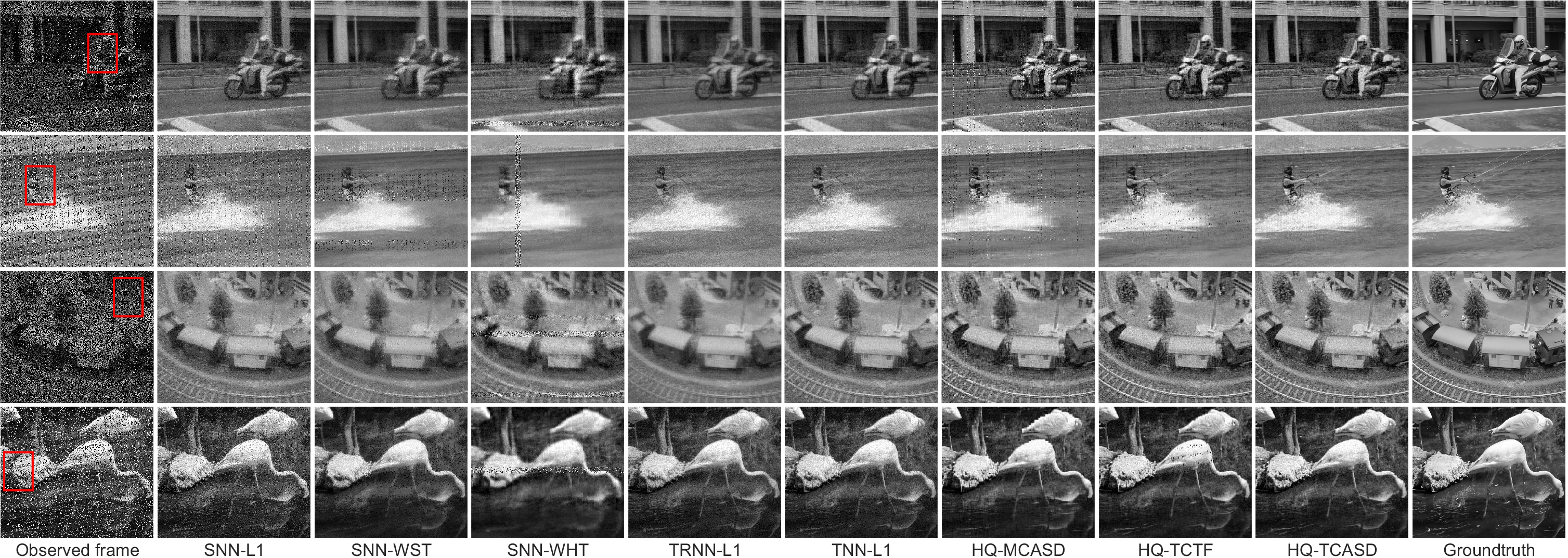}
\caption{{Frames recovered by different algorithms under different noise distributions with $c=0.2$.}}
\label{fig:video_large}
\end{figure*}

\begin{figure*}[htbp]
\centering
\includegraphics[width=1\textwidth]{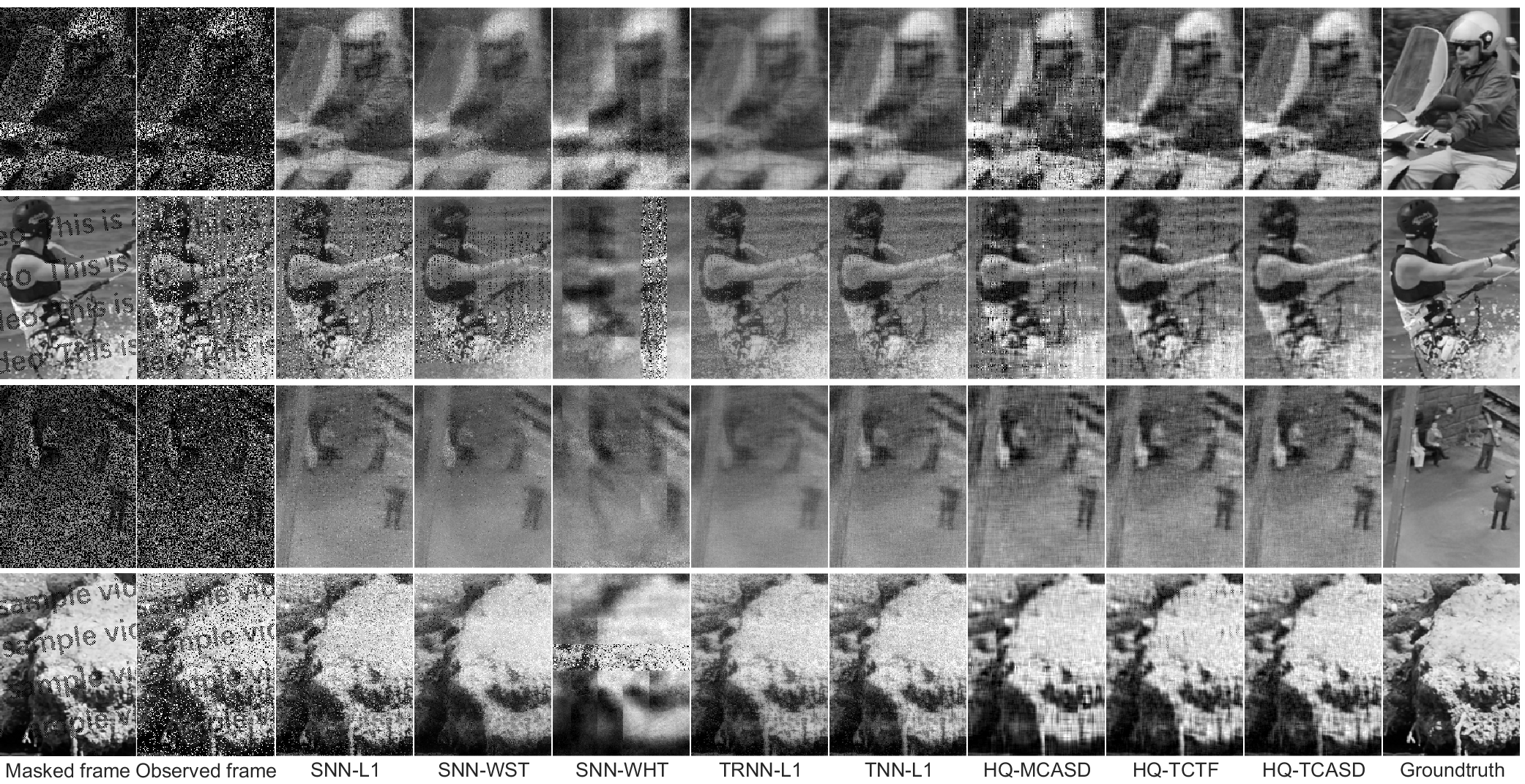}
\caption{{Enlarged regions (red rectangles of Fig.\ref{fig:video_large}) of recovered frames by different algorithms.}}
\label{fig:video_small}
\end{figure*}

Similar to the image inpainting task, we compare performance under different missing pixel patterns and noise distributions. The curves of average PSNR for different values of $c$ are shown in Fig.~\ref{fig:video_curves}. The corresponding average running times are depicted in Fig.~\ref{fig:video_time}. {The proposed HQ-TCASD algorithm achieves the highest PSNR values in most situations, and HQ-TCASD achieves a 3-fold speedup over other algorithms using parallel computation.} 
To shed more light on performance, Fig.~\ref{fig:video_large} illustrates examples of recovered video frames from four video sequences. In Fig.~\ref{fig:video_small}, we zoom in on the regions of Fig.~\ref{fig:video_large} surrounded by the red rectangles. 
It can be seen that HQ-TCASD yields frames that are less noisy and with better contrast than the ones recovered by the other methods.

{We also investigate the performance with an increasing number of video frames. The `train' video with GMM noise $c=0.2$ is utilized in this experiment. The video length is increased from $1$ to $50$ and the corresponding average PSNR curves of different algorithms are shown in Fig.~\ref{fig:video_theory} (right). As shown, at first, the average PSNR increases rapidly as the number of frames increases. Then, the average PSNR of all algorithm remains unchanged or slightly decreases except SNN-WHT and SNN-L1. To better understand the tubal rank property with an increasing number of frames, we set the tubal multi-rank to the same value $r$ and compute the PSNR for the best $r$-tubal-rank approximation of the original video. The results are shown in Fig.~\ref{fig:video_theory} (left). It can be seen that the performance only degrades slightly as the number of frames increases. Therefore, one can use a fixed setting of the tubal rank for different number of frames, which is also verified in Fig.~\ref{fig:video_theory} (right).}

\begin{figure}[htbp]
\centering
\includegraphics[width=0.48\textwidth]{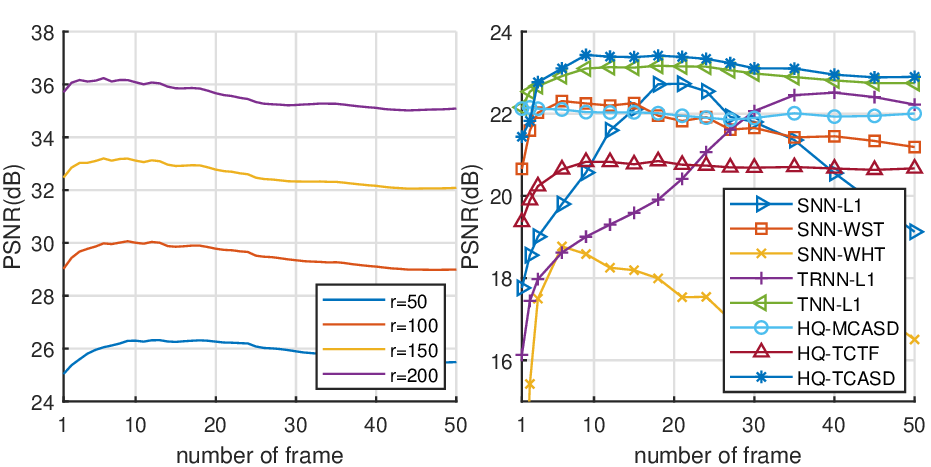}
\caption{{Left: PSNR curves of best $r$-tubal-rank approximation of the original video with a different number of frames. Right: PSNR curves of different algorithms versus number of frames.}}
\label{fig:video_theory}
\end{figure}

\subsection{Traffic Data Prediction}
In this section, we further evaluate the performance of the algorithms using traffic data. The traffic data is generated from the Large-scale PeMS traffic speed dataset\footnote{https://doi.org/10.5281/zenodo.3939793} \cite{mallick2020transfer}. The data registers traffic speed time series from $11160$ sensors over $4$ weeks with $288$ time points per day (i.e., 5-min frequency) in California, USA. Thus it forms a $11160\times288\times28$ tensor. Each value of the data is normalized such that all data are in the range $[0,1]$. In this experiment, we randomly and uniformly selected $50\%$ of the data points as the observed data. The noise parameter $\sigma_A^2$ is set to zero and the outliers have $\sigma_B^2=1$. For HQ-TCASD and HQ-TCTF, the elements of the multi-rank vector are all fixed at $20$. $20$ Monte Calro runs are performed for each value of $c$ with different selections of observed data and noise. The values of the average relative error under different simulation settings are reported in Table \ref{tb:traffic}. {The running time in seconds of HQ-TCASD and HQ-TCTF using parallel computation is shown between brackets.} HQ-TCASD achieves the best performance for $c=0.2$ and $0.3$. To better illustrate the recovery performance, an example of the data recovered from sensor No. $9960$ on the $26$-th day under $c=0.3$ is depicted in Fig.~\ref{fig:traffic}. It can be seen that the proposed HQ-TCASD outperforms the other algorithms.

\renewcommand\arraystretch{1.1}
\begin{table*}[]
\caption{Completion performance (relative error) comparison on traffic data.}
\label{tb:traffic}
\resizebox{\textwidth}{!}{%
\begin{tabular}{@{}cccccccccc@{}}
\toprule
Parameter              & Metric  & SNN-L1  & SNN-WST & SNN-WHT & TRNN-L1 & TNN-L1       & HQ-MCASD        & HQ-TCTF        & HQ-TCASD        \\ \midrule
\multirow{2}{*}{$c=0.1$} & \textit{rel.err} & 0.0673  & 0.0705  & 0.0753  & 0.0489  & 0.0436       & \textbf{0.0377} & {\ul 0.0385}   & 0.0397          \\
                       & time(s) & 14437.7 & 8145.53 & 6891.92 & 12208.6 & 568.65       & 189.78          & 316.56 (223.52) & 959.25 (370.15)  \\
\multirow{2}{*}{$c=0.2$} & \textit{rel.err} & 0.0692  & 0.0728  & 0.0809  & 0.0511  & 0.0466       & 0.0575          & {\ul 0.0451}   & \textbf{0.0423} \\
                       & time(s) & 10440.9 & 9423.65 & 6745.36 & 13850.3 & 710.64       & 311.89          & 351.35 (243.23) & 817.58 (327.38)  \\
\multirow{2}{*}{$c=0.3$} & \textit{rel.err} & 0.0726  & 0.0755  & 0.0935  & 0.0536  & {\ul 0.0527} & 0.1179          & 0.0560         & \textbf{0.0476} \\
                       & time(s) & 13092.0 & 10932.2 & 6328.92 & 14832.2 & 1011.82      & 618.94          & 396.08 (269.25) & 738.43 (291.34)  \\ \bottomrule
\end{tabular}%
}
\end{table*}

\begin{figure}[ht]
\centering
\includegraphics[width=0.5\textwidth]{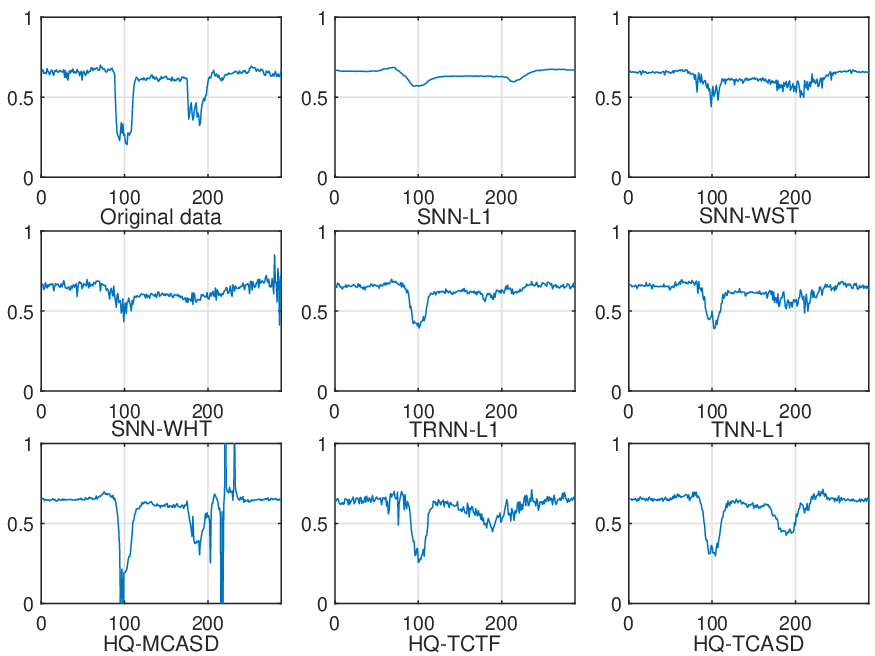}
\caption{Examples of the recovered missing signals of traffic data.}
\label{fig:traffic}
\end{figure}

\section{Conclusion}
\label{sec:conc}
In this paper, we proposed a novel robust tensor completion method that utilizes tensor-factorization to impose a low-tubal-rank structure, which avoids the computation of the SVD. The correntropy measure is introduced to alleviate the impact of large outliers. Based on a half-quadratic minimization technique, two efficient robust tensor completion algorithms, HQ-TCTF and HQ-TCASD, were proposed and their convergence is analyzed. Experiments on both synthetic and real datasets demonstrate the superior performance of the proposed methods compared to existing state-of-the-art algorithms.

\begin{appendices}
\section{Proof of Proposition 2}
Since $\boldsymbol{\cal{W}}$ is an optimal solution for \eqref{eq:HQW}, we have
\begin{equation}
J({\boldsymbol{\cal{X}}^{t+1},\boldsymbol{\cal{Y}}}^{t+1},\boldsymbol{\cal{W}}^{t+1})\leq J({\boldsymbol{\cal{X}}^{t+1},\boldsymbol{\cal{Y}}}^{t+1},\boldsymbol{\cal{W}}^{t}) \:.
\end{equation}
By fixing $\boldsymbol{\cal{W}}$ and defining $\boldsymbol{Q}=\sqrt{\tilde{\boldsymbol{W}}} \circ \tilde{\boldsymbol{P}}$, we obtain the following \begin{equation}
\begin{aligned}
\label{eq:ASD_p1_g}
&2J(\boldsymbol{U}^{t+1},\hat{\boldsymbol{Y}}^{t})-2J(\boldsymbol{U}^{t},\hat{\boldsymbol{Y}}^{t})\\
&={\left\|{\boldsymbol{Q} \circ \left({\tilde{\boldsymbol{M}}}-{\boldsymbol{U}^{t+1}{\hat{{\boldsymbol{Y}}}^{t}}}\right)}\right\|_{F}^2}-{\left\|{\boldsymbol{Q} \circ \left({\tilde{\boldsymbol{M}}}-{\boldsymbol{U}^{t}{\hat{{\boldsymbol{Y}}}^{t}}}\right)}\right\|_{F}^2}\\
&={\left\|{\boldsymbol{Q} \!\circ\! \left({\tilde{\boldsymbol{M}}}\!-\!({\boldsymbol{U}^{t}-\mu'^t_{\boldsymbol{U}}{\boldsymbol{g}}^{'t}_{\boldsymbol{U}}){\hat{{\boldsymbol{Y}}}^{t}}}\right)}\right\|_{F}^2}\!-\!{\left\|{\boldsymbol{Q} \!\circ\! \left({\tilde{\boldsymbol{M}}}\!-\!{\boldsymbol{U}^{t}{\hat{{\boldsymbol{Y}}}^{t}}}\right)}\right\|_{F}^2}\\
&=\!(\mu'^t_{\boldsymbol{U}})^2{\left\|{\boldsymbol{Q} \!\circ\! \left({{\boldsymbol{g}}^{'t}_{\boldsymbol{U}}{\hat{{\boldsymbol{Y}}}^{t}}}\right)}\right\|_{F}^2}\!+2\mu'^t_{\boldsymbol{U}}\!\left<{\!\boldsymbol{Q} \!\circ\! \left({\tilde{\boldsymbol{M}}}\!-\!{\boldsymbol{U}^{t}{\hat{{\boldsymbol{Y}}}^{t}}}\right)}, {{\boldsymbol{g}}^{'t}_{\boldsymbol{U}}{\hat{{\boldsymbol{Y}}}^{t}}}\right>\\
&=\frac{(\|\boldsymbol{g}^{'t}_{{{\boldsymbol{U}}}}\|_F^2)^2}{\left\|\boldsymbol{Q}\circ\left( {\boldsymbol{g}}^{'t}_{{\boldsymbol{U}}}{{\hat{\boldsymbol{Y}}}}\right)\right\|_F^2}-2\mu'^t_{\boldsymbol{U}}\left<{\boldsymbol{g}}^{t}_{\boldsymbol{U}}, {{\boldsymbol{g}}^{'t}_{\boldsymbol{U}}}\right>
\end{aligned}
\end{equation}
We can further simplify $\langle{\boldsymbol{g}}_{\boldsymbol{U}}, {{\boldsymbol{g}}^{'}_{\boldsymbol{U}}}\rangle$ as
\begin{equation}
\begin{aligned}
\label{eq:ASD_p2_g1}
\left<{\boldsymbol{g}}_{\boldsymbol{U}}, {{\boldsymbol{g}}'_{\boldsymbol{U}}}\right>=&\operatorname{tr}\left(\boldsymbol{g}^{*}_{\boldsymbol{U}}\boldsymbol{F}^{-1}\operatorname{bdiagz}(\boldsymbol{F}\boldsymbol{g}_{\boldsymbol{U}})\right)\\
=&\frac{1}{n_3}\operatorname{tr}\left((\boldsymbol{F}\boldsymbol{g}_{\boldsymbol{U}})^{*}\operatorname{bdiagz}(\boldsymbol{F}\boldsymbol{g}_{\boldsymbol{U}})\right)\\
=&\frac{1}{n_3}\left\|\operatorname{bdiagz}(\boldsymbol{F}\boldsymbol{g}_{\boldsymbol{U}})\right\|_F^2
\end{aligned}
\end{equation}
where $\operatorname{tr}(\cdot)$ denotes the trace operator. Further, $\|\boldsymbol{g}'_{{{\boldsymbol{U}}}}\|_F^2$ can be simplified as
\begin{equation}
\begin{aligned}
\label{eq:ASD_p2_g2}
\|\boldsymbol{g}'_{{{\boldsymbol{U}}}}\|_F^2=\frac{1}{n_3}\|\boldsymbol{F}\operatorname{bdiagz}(\boldsymbol{F}\boldsymbol{g}_{\boldsymbol{U}})\|_F^2=\frac{1}{n_3}\|\operatorname{bdiagz}(\boldsymbol{F}\boldsymbol{g}_{\boldsymbol{U}})\|_F^2
\end{aligned}
\end{equation}
where we use the fact that $\boldsymbol{F}^{*}\boldsymbol{F}=\boldsymbol{I}$. Therefore, according to \eqref{eq:ASD_p2_g1} and \eqref{eq:ASD_p2_g2} we have
\begin{equation}
\begin{aligned}
\|\boldsymbol{g}'_{{{\boldsymbol{U}}}}\|_F^2=\langle{\boldsymbol{g}}_{\boldsymbol{U}}, {{\boldsymbol{g}}^{'}_{\boldsymbol{U}}}\rangle
\end{aligned}
\end{equation}
and \eqref{eq:ASD_p1_g} can be written as
\begin{equation}
\label{eq:ASD_p2_g3}
J(\boldsymbol{U}^{t+1},\hat{\boldsymbol{Y}}^{t})-J(\boldsymbol{U}^{t},\hat{\boldsymbol{Y}}^{t})=-\frac{(\|\boldsymbol{g}^{'t}_{{{\boldsymbol{U}}}}\|_F^2)^2}{2\left\|\boldsymbol{Q}\circ\left( {\boldsymbol{g}}^{'t}_{{\boldsymbol{U}}}{{\hat{\boldsymbol{Y}}}}\right)\right\|_F^2}\leq 0
\end{equation}
Similarly, we can obtain
\begin{equation}
\begin{aligned}
\label{eq:ASD_p2_g4}
&J(\boldsymbol{U}^{t+1},\hat{\boldsymbol{Y}}^{t+1})-J(\boldsymbol{U}^{t+1},\hat{\boldsymbol{Y}}^{t})\\
&=-(1\!-\!\lambda)\frac{(\|\boldsymbol{g}^t_{{\hat{\boldsymbol{Y}}}}\|_F^2)^2}{2\left\|{\boldsymbol{Q}} \!\circ\! \left({\boldsymbol{U}^{t+1}}{\boldsymbol{g}}^t_{{\hat{\boldsymbol{Y}}}}\right)\right\|_F^2}\!-\!\lambda\frac{|\langle\boldsymbol{g}^t_{{\hat{\boldsymbol{Y}}}},\boldsymbol{g}'^t_{{\hat{\boldsymbol{Y}}}}\rangle|^2}{2\left\|{\boldsymbol{Q}} \!\circ\! \left({\boldsymbol{U}^{t+1}}{\boldsymbol{g}}'^t_{{\hat{\boldsymbol{Y}}}}\right)\right\|_F^2}\\
&\leq 0
\end{aligned}
\end{equation}
\eqref{eq:ASD_p2_g3} and \eqref{eq:ASD_p2_g4} imply that
\begin{equation}
J(\boldsymbol{U}^{t+1},\hat{\boldsymbol{Y}}^{t+1})\leq J(\boldsymbol{U}^{t},\hat{\boldsymbol{Y}}^{t})
\end{equation}
Thus, according to the relation between $\boldsymbol{U}$ and $\boldsymbol{\cal{X}}$, ${\hat{\boldsymbol{{ Y }}}}$ and $\boldsymbol{\cal{Y}}$, we have that
\begin{equation}
J({\boldsymbol{\cal{X}}^{t+1},\boldsymbol{\cal{Y}}}^{t+1},\boldsymbol{\cal{W}}^{t+1})\leq J({\boldsymbol{\cal{X}}^{t},\boldsymbol{\cal{Y}}}^{t},\boldsymbol{\cal{W}}^{t})\:.
\end{equation}
It can be also verified that $J({\boldsymbol{\cal{X}}^{t},\boldsymbol{\cal{Y}}}^{t},\boldsymbol{\cal{W}}^{t})$ is always bounded below for arbitrary $t$. Thus, $\{J({\boldsymbol{\cal{X}}^{t},\boldsymbol{\cal{Y}}}^{t},\boldsymbol{\cal{W}}^{t}),t=1,2,...\}$ will converge.
\end{appendices}

\bibliographystyle{IEEEtran}
\bibliography{references}

\begin{thebibliography}{10}
\providecommand{\url}[1]{#1}
\csname url@samestyle\endcsname
\providecommand{\newblock}{\relax}
\providecommand{\bibinfo}[2]{#2}
\providecommand{\BIBentrySTDinterwordspacing}{\spaceskip=0pt\relax}
\providecommand{\BIBentryALTinterwordstretchfactor}{4}
\providecommand{\BIBentryALTinterwordspacing}{\spaceskip=\fontdimen2\font plus
\BIBentryALTinterwordstretchfactor\fontdimen3\font minus
  \fontdimen4\font\relax}
\providecommand{\BIBforeignlanguage}[2]{{%
\expandafter\ifx\csname l@#1\endcsname\relax
\typeout{** WARNING: IEEEtran.bst: No hyphenation pattern has been}%
\typeout{** loaded for the language `#1'. Using the pattern for}%
\typeout{** the default language instead.}%
\else
\language=\csname l@#1\endcsname
\fi
#2}}
\providecommand{\BIBdecl}{\relax}
\BIBdecl

\bibitem{dian2017hyperspectral}
R.~Dian, L.~Fang, and S.~Li, ``Hyperspectral image super-resolution via
  non-local sparse tensor factorization,'' in \emph{Proceedings of the IEEE
  Conference on Computer Vision and Pattern Recognition}, 2017, pp. 5344--5353.

\bibitem{zhang2019computational}
S.~Zhang, L.~Wang, Y.~Fu, X.~Zhong, and H.~Huang, ``Computational hyperspectral
  imaging based on dimension-discriminative low-rank tensor recovery,'' in
  \emph{Proceedings of the IEEE International Conference on Computer Vision},
  2019, pp. 10\,183--10\,192.

\bibitem{xie2018unifying}
Y.~Xie, D.~Tao, W.~Zhang, Y.~Liu, L.~Zhang, and Y.~Qu, ``On unifying multi-view
  self-representations for clustering by tensor multi-rank minimization,''
  \emph{International Journal of Computer Vision}, vol. 126, no.~11, pp.
  1157--1179, 2018.

\bibitem{zhao2019multi}
T.~Zhao, Y.~Xu, M.~Monfort, W.~Choi, C.~Baker, Y.~Zhao, Y.~Wang, and Y.~N. Wu,
  ``Multi-agent tensor fusion for contextual trajectory prediction,'' in
  \emph{Proceedings of the IEEE Conference on Computer Vision and Pattern
  Recognition}, 2019, pp. 12\,126--12\,134.

\bibitem{he2006tensor}
X.~He, D.~Cai, and P.~Niyogi, ``Tensor subspace analysis,'' in \emph{Advances
  in Neural Information Processing Systems}, 2006, pp. 499--506.

\bibitem{xie2020robust}
Y.~Xie, J.~Liu, Y.~Qu, D.~Tao, W.~Zhang, L.~Dai, and L.~Ma, ``Robust kernelized
  multiview self-representation for subspace clustering,'' \emph{IEEE
  Transactions on Neural Networks and Learning Systems}, vol.~32, no.~2, pp.
  868--881, 2020.

\bibitem{tang2021one}
Y.~Tang, Y.~Xie, C.~Zhang, Z.~Zhang, and W.~Zhang, ``One-step multiview
  subspace segmentation via joint skinny tensor learning and latent
  clustering,'' \emph{IEEE Transactions on Cybernetics}, 2021.

\bibitem{sidiropoulos2017tensor}
N.~D. Sidiropoulos, L.~De~Lathauwer, X.~Fu, K.~Huang, E.~E. Papalexakis, and
  C.~Faloutsos, ``Tensor decomposition for signal processing and machine
  learning,'' \emph{IEEE Transactions on Signal Processing}, vol.~65, no.~13,
  pp. 3551--3582, 2017.

\bibitem{cichocki2015tensor}
A.~Cichocki, D.~Mandic, L.~De~Lathauwer, G.~Zhou, Q.~Zhao, C.~Caiafa, and H.~A.
  Phan, ``Tensor decompositions for signal processing applications: From
  two-way to multiway component analysis,'' \emph{IEEE Signal Processing
  Magazine}, vol.~32, no.~2, pp. 145--163, 2015.

\bibitem{tang2019tensor}
Y.~Tang, Y.~Xie, X.~Yang, J.~Niu, and W.~Zhang, ``Tensor multi-elastic kernel
  self-paced learning for time series clustering,'' \emph{IEEE Transactions on
  Knowledge and Data Engineering}, 2019.

\bibitem{kiers2000towards}
H.~A. Kiers, ``Towards a standardized notation and terminology in multiway
  analysis,'' \emph{Journal of Chemometrics: A Journal of the Chemometrics
  Society}, vol.~14, no.~3, pp. 105--122, 2000.

\bibitem{tucker1966some}
L.~R. Tucker, ``Some mathematical notes on three-mode factor analysis,''
  \emph{Psychometrika}, vol.~31, no.~3, pp. 279--311, 1966.

\bibitem{zhao2016tensor}
Q.~Zhao, G.~Zhou, S.~Xie, L.~Zhang, and A.~Cichocki, ``Tensor ring
  decomposition,'' \emph{arXiv preprint arXiv:1606.05535}, 2016.

\bibitem{kilmer2011factorization}
M.~E. Kilmer and C.~D. Martin, ``Factorization strategies for third-order
  tensors,'' \emph{Linear Algebra and its Applications}, vol. 435, no.~3, pp.
  641--658, 2011.

\bibitem{candes2009exact}
E.~J. Cand{\`e}s and B.~Recht, ``Exact matrix completion via convex
  optimization,'' \emph{Foundations of Computational mathematics}, vol.~9,
  no.~6, p. 717, 2009.

\bibitem{candes2010matrix}
E.~J. Candes and Y.~Plan, ``Matrix completion with noise,'' \emph{Proceedings
  of the IEEE}, vol.~98, no.~6, pp. 925--936, 2010.

\bibitem{jain2014provable}
P.~Jain and S.~Oh, ``Provable tensor factorization with missing data,'' in
  \emph{Advances in Neural Information Processing Systems}, 2014, pp.
  1431--1439.

\bibitem{kasai2016low}
H.~Kasai and B.~Mishra, ``Low-rank tensor completion: a riemannian manifold
  preconditioning approach,'' in \emph{International Conference on Machine
  Learning}, 2016, pp. 1012--1021.

\bibitem{xu2015parallel}
Y.~Xu, R.~Hao, W.~Yin, and Z.~Su, ``Parallel matrix factorization for low-rank
  tensor completion,'' \emph{Inverse Problems and Imaging}, vol.~9, no.~2, pp.
  601--624, 2015.

\bibitem{zhang2016exact}
Z.~Zhang and S.~Aeron, ``Exact tensor completion using t-svd,'' \emph{IEEE
  Transactions on Signal Processing}, vol.~65, no.~6, pp. 1511--1526, 2016.

\bibitem{zhou2017tensor}
P.~Zhou, C.~Lu, Z.~Lin, and C.~Zhang, ``Tensor factorization for low-rank
  tensor completion,'' \emph{IEEE Transactions on Image Processing}, vol.~27,
  no.~3, pp. 1152--1163, 2017.

\bibitem{liu2019low}
X.-Y. Liu, S.~Aeron, V.~Aggarwal, and X.~Wang, ``Low-tubal-rank tensor
  completion using alternating minimization,'' \emph{IEEE Transactions on
  Information Theory}, vol.~66, no.~3, pp. 1714--1737, 2019.

\bibitem{gilman2020grassmannian}
K.~Gilman and L.~Balzano, ``Grassmannian optimization for online tensor
  completion and tracking in the t-svd algebra,'' \emph{arXiv preprint
  arXiv:2001.11419}, 2020.

\bibitem{goldfarb2014robust}
D.~Goldfarb and Z.~Qin, ``Robust low-rank tensor recovery: Models and
  algorithms,'' \emph{SIAM Journal on Matrix Analysis and Applications},
  vol.~35, no.~1, pp. 225--253, 2014.

\bibitem{han2018generalized}
Z.~Han, Y.~Wang, Q.~Zhao, D.~Meng, L.~Lin, Y.~Tang \emph{et~al.}, ``A
  generalized model for robust tensor factorization with noise modeling by
  mixture of gaussians,'' \emph{IEEE Transactions on Neural Networks and
  Learning Systems}, vol.~29, no.~11, pp. 5380--5393, 2018.

\bibitem{inoue2009robust}
K.~Inoue, K.~Hara, and K.~Urahama, ``Robust multilinear principal component
  analysis,'' in \emph{2009 IEEE 12th International Conference on Computer
  Vision}.\hskip 1em plus 0.5em minus 0.4em\relax IEEE, 2009, pp. 591--597.

\bibitem{huang2014provable}
B.~Huang, C.~Mu, D.~Goldfarb, and J.~Wright, ``Provable low-rank tensor
  recovery,'' \emph{Optimization-Online}, vol. 4252, no.~2, pp. 455--500, 2014.

\bibitem{yang2015robust}
Y.~Yang, Y.~Feng, and J.~A. Suykens, ``Robust low-rank tensor recovery with
  regularized redescending m-estimator,'' \emph{IEEE Transactions on Neural
  Networks and Learning Systems}, vol.~27, no.~9, pp. 1933--1946, 2015.

\bibitem{huang2020robust}
H.~Huang, Y.~Liu, Z.~Long, and C.~Zhu, ``Robust low-rank tensor ring
  completion,'' \emph{IEEE Transactions on Computational Imaging}, vol.~6, pp.
  1117--1126, 2020.

\bibitem{kilmer2013third}
M.~E. Kilmer, K.~Braman, N.~Hao, and R.~C. Hoover, ``Third-order tensors as
  operators on matrices: A theoretical and computational framework with
  applications in imaging,'' \emph{SIAM Journal on Matrix Analysis and
  Applications}, vol.~34, no.~1, pp. 148--172, 2013.

\bibitem{jiang2019robust}
Q.~Jiang and M.~Ng, ``Robust low-tubal-rank tensor completion via convex
  optimization.'' in \emph{IJCAI}, 2019, pp. 2649--2655.

\bibitem{wang2019robust}
A.~Wang, X.~Song, X.~Wu, Z.~Lai, and Z.~Jin, ``Robust low-tubal-rank tensor
  completion,'' in \emph{IEEE International Conference on Acoustics, Speech and
  Signal Processing (ICASSP)}, 2019, pp. 3432--3436.

\bibitem{liu2007correntropy}
W.~Liu, P.~P. Pokharel, and J.~C. Pr{\'\i}ncipe, ``Correntropy: Properties and
  applications in non-gaussian signal processing,'' \emph{IEEE Transactions on
  Signal Processing}, vol.~55, no.~11, pp. 5286--5298, 2007.

\bibitem{chen2016generalized}
B.~Chen, L.~Xing, H.~Zhao, N.~Zheng, J.~C. Pr{\i} \emph{et~al.}, ``Generalized
  correntropy for robust adaptive filtering,'' \emph{IEEE Transactions on
  Signal Processing}, vol.~64, no.~13, pp. 3376--3387, 2016.

\bibitem{he2010maximum}
R.~He, W.-S. Zheng, and B.-G. Hu, ``Maximum correntropy criterion for robust
  face recognition,'' \emph{IEEE Transactions on Pattern Analysis and Machine
  Intelligence}, vol.~33, no.~8, pp. 1561--1576, 2010.

\bibitem{he2019robust}
Y.~He, F.~Wang, Y.~Li, J.~Qin, and B.~Chen, ``Robust matrix completion via
  maximum correntropy criterion and half-quadratic optimization,'' \emph{IEEE
  Transactions on Signal Processing}, vol.~68, pp. 181--195, 2019.

\bibitem{zheng2020broad}
Y.~Zheng, B.~Chen, S.~Wang, and W.~Wang, ``Broad learning system based on
  maximum correntropy criterion,'' \emph{IEEE Transactions on Neural Networks
  and Learning Systems}, 2020.

\bibitem{chen2017maximum}
B.~Chen, X.~Liu, H.~Zhao, and J.~C. Principe, ``Maximum correntropy kalman
  filter,'' \emph{Automatica}, vol.~76, pp. 70--77, 2017.

\bibitem{zhao2011kernel}
S.~Zhao, B.~Chen, and J.~C. Principe, ``Kernel adaptive filtering with maximum
  correntropy criterion,'' in \emph{International Joint Conference on Neural
  Networks}, 2011, pp. 2012--2017.

\bibitem{nikolova2005analysis}
M.~Nikolova and M.~K. Ng, ``Analysis of half-quadratic minimization methods for
  signal and image recovery,'' \emph{SIAM Journal on Scientific computing},
  vol.~27, no.~3, pp. 937--966, 2005.

\bibitem{lu2019tensor}
C.~Lu, J.~Feng, Y.~Chen, W.~Liu, Z.~Lin, and S.~Yan, ``Tensor robust principal
  component analysis with a new tensor nuclear norm,'' \emph{IEEE Transactions
  on Pattern Analysis and Machine Intelligence}, vol.~42, no.~4, pp. 925--938,
  2019.

\bibitem{hu2016twist}
W.~Hu, D.~Tao, W.~Zhang, Y.~Xie, and Y.~Yang, ``The twist tensor nuclear norm
  for video completion,'' \emph{IEEE Transactions on Neural Networks and
  Learning Systems}, vol.~28, no.~12, pp. 2961--2973, 2016.

\bibitem{haldar2009rank}
J.~P. Haldar and D.~Hernando, ``Rank-constrained solutions to linear matrix
  equations using powerfactorization,'' \emph{IEEE Signal Processing Letters},
  vol.~16, no.~7, pp. 584--587, 2009.

\bibitem{Singh2009Using}
A.~Singh and J.~C. Principe, ``Using correntropy as a cost function in linear
  adaptive filters,'' in \emph{IEEE International Joint Conference on Neural
  Networks}, 2009, pp. 2950--2955.

\bibitem{singh2014c}
A.~Singh, R.~Pokharel, and J.~Principe, ``The c-loss function for pattern
  classification,'' \emph{Pattern Recognition}, vol.~47, no.~1, pp. 441--453,
  2014.

\bibitem{dennis1978techniques}
J.~E. Dennis~Jr and R.~E. Welsch, ``Techniques for nonlinear least squares and
  robust regression,'' \emph{Communications in Statistics-simulation and
  Computation}, vol.~7, no.~4, pp. 345--359, 1978.

\bibitem{tanner2016low}
J.~Tanner and K.~Wei, ``Low rank matrix completion by alternating steepest
  descent methods,'' \emph{Applied and Computational Harmonic Analysis},
  vol.~40, no.~2, pp. 417--429, 2016.

\bibitem{perazzi2016benchmark}
F.~Perazzi, J.~Pont-Tuset, B.~McWilliams, L.~Van~Gool, M.~Gross, and
  A.~Sorkine-Hornung, ``A benchmark dataset and evaluation methodology for
  video object segmentation,'' in \emph{Proceedings of the IEEE Conference on
  Computer Vision and Pattern Recognition}, 2016, pp. 724--732.

\bibitem{mallick2020transfer}
T.~Mallick, P.~Balaprakash, E.~Rask, and J.~Macfarlane, ``Transfer learning
  with graph neural networks for short-term highway traffic forecasting,''
  \emph{arXiv preprint arXiv:2004.08038}, 2020.

\end{thebibliography}
\end{document}